\documentclass[preprint, 12pt, authoyear]{article}
\usepackage{jmlr2e} 

\usepackage[english]{babel}
\usepackage[latin1]{inputenc}
\usepackage{amsmath,amssymb,latexsym,amsfonts,lipsum}
\usepackage{bm}
\usepackage{mathtools}
\usepackage{graphics}

\usepackage{algorithm}
\usepackage{algorithmic}
\usepackage{caption}
\usepackage{color}
\usepackage{float}

\usepackage{graphicx} \usepackage{verbatim}
\usepackage[margin=1in]{geometry}
\usepackage[font=footnotesize,labelfont=footnotesize]{caption}
\usepackage[font=scriptsize,labelfont=scriptsize]{subcaption}
\usepackage{mathtools}
\usepackage{enumerate}
\allowdisplaybreaks

\newcommand{\vol}{\text{vol}}
\newcommand{\D}{\bold{D}}
\newcommand{\GammaMin}{\Gamma_{\text{min}}}

\newcommand{\Dt}{\mathcal{D}_{t}}
\newcommand{\I}{I}
\newcommand{\K}{\mathcal{K}}
\newcommand{\M}{\mathcal{M}}
\newcommand{\Ncut}{\text{Ncut}}
\newcommand{\Pm}{\bold{P}}
\newcommand{\Sm}{\bold{S}}

\newcommand{\Z}{\bold{Z}}

\newcommand{\s}{\bold{s}}
\newcommand{\sinf}{\bold{s}^{\infty}}
\newcommand{\x}{\textbf{x}}

\newcommand{\boldpi}{\bm{\pi}}
\newcommand{\0}{\bm{0}}

\newcommand{\numclust}{K} 
\newcommand{\Dbtw}{D_{t}^{\text{btw}}}
\newcommand{\Din}{D_{t}^{\text{in}}}
\newcommand{\Deg}{{D}}  
\newcommand{\rhoEuc}{\rho^{\text{Euc}}}
\newcommand{\G}{\mathcal{G}}
\newcommand{\knn}{k_{\text{nn}}}

\DeclareMathOperator*{\argmax}{arg\,max}
\DeclareMathOperator*{\cut}{\text{cut}}
\DeclareMathOperator*{\Lsym}{L_{SYM}}
\DeclareMathOperator*{\Tr}{\text{Tr}}

\usepackage{hyperref}

\newtheorem{thm}{Theorem}[section]
\newtheorem{cor}[thm]{Corollary}

\newtheorem{defn}[thm]{Definition}

\DeclareMathOperator*{\argmin}{arg\,min}

\makeatletter
\newcommand{\distas}[1]{\mathbin{\overset{#1}{\kern\z@\sim}}}%
\newsavebox{\mybox}\newsavebox{\mysim}
\newcommand{\distras}[1]{%
  \savebox{\mybox}{\hbox{\kern3pt$\scriptstyle#1$\kern3pt}}%
  \savebox{\mysim}{\hbox{$\sim$}}%
  \mathbin{\overset{#1}{\kern\z@\resizebox{\wd\mybox}{\ht\mysim}{$\sim$}}}%
}
\makeatother

\makeatletter
\let\c@equation\c@thm
\makeatother
\numberwithin{equation}{section}

\makeatletter
\def\ps@pprintTitle{%
 \let\@oddhead\@empty
 \let\@evenhead\@empty
 \def\@oddfoot{}%
 \let\@evenfoot\@oddfoot}
\makeatother

\begin{document}

\title{Learning by Unsupervised Nonlinear Diffusion}

\author{\name Mauro Maggioni \email mauro.maggioni@jhu.edu \\
       \addr Department of Mathematics, Department of Applied Mathematics and Statistics,\\
       Mathematical Institute of Data Sciences, Institute of Data Intensive Engineering and Science,\\
       Johns Hopkins University, Baltimore, MD 21218, USA
       \AND 
         \name James M. Murphy \email jm.murphy@tufts.edu \\
        \addr  Department of Mathematics\\
       Tufts University, Medford, MA 02155, USA\\
       }

\editor{}

\maketitle

\begin{abstract}This paper proposes and analyzes a novel clustering algorithm that combines graph-based diffusion geometry with techniques based on density and mode estimation.  The proposed method is suitable for data generated from mixtures of distributions with densities that are both multimodal and have nonlinear shapes.  A crucial aspect of this algorithm is the use of time of a data-adapted diffusion process as a scale parameter that is different from the local spatial scale parameter used in many clustering algorithms. We prove estimates for the behavior of diffusion distances with respect to this time parameter under a flexible nonparametric data model, identifying a range of times in which the mesoscopic equilibria of the underlying process are revealed, corresponding to a gap between within-cluster and between-cluster diffusion distances. These structures can be missed by the top eigenvectors of the graph Laplacian, commonly used in spectral clustering. This analysis is leveraged to prove sufficient conditions guaranteeing the accuracy of the proposed \emph{learning by unsupervised nonlinear diffusion (LUND)} procedure.  We implement LUND and confirm its theoretical properties on illustrative datasets, demonstrating the theoretical and empirical advantages over both spectral clustering and density-based clustering techniques.  

\end{abstract}


\section{Introduction}

Unsupervised learning is a central problem in machine learning, requiring that data be analyzed without a priori knowledge of any class labels.  One of the most common unsupervised problems is the problem of \emph{clustering}, in which the data is to be partitioned into clusters so that each cluster contains similar points and distinct clusters are sufficiently separated.  In general, this problem is ill-posed, requiring various geometric, analytic, topological, and statistical assumptions on the data and measurement method be imposed to make it tractable.  Developing conditions under which empirically effective clustering methods have performance guarantees is an active research topic \citep{Hartigan1981, Ng2002, VonLuxburg2008, Awasthi2015, Lei2015, Schiebinger2015, Trillos2016_1, Trillos2016_2, Little2017Path}, as is the development of broad analyses and characterizations of generic clustering methods \citep{Ackerman2013, Ben-David2015}.  

Clustering techniques abound.  Some of the most popular include $\numclust$-means clustering and its variants \citep{Friedman2001, Arthur2007, Park2009}, hierarchical methods \citep{Hartigan1985, Friedman2001}, density-based methods \citep{Ester1996}, and mode-based methods \citep{Fukunaga1975, Comaniciu2002, Chacon2012, Rodriguez2014, Genovese2016}.  Feature extraction is often combined with these standard methods to improve clustering performance.  In particular, \emph{spectral methods} \citep{Shi2000, Ng2002,Coifman2005,VonLuxburg2007} construct graphs representing data, and use the spectral properties of the resulting weight matrix or Laplacian to produce structure-revealing features in the data.  

Graphs often encode pairwise similarities between points, typically local: this ``spatial'' scale is often determined by a parameter $\sigma$. For example only points within distance $4\sigma$ of each other may be connected, with weight $e^{-\|x_i-x_j\|_{2}^2/\sigma^2}$. From the graph, global features on the data may be derived if needed, for example by considering the eigenfunctions of the random walk on the graph.  Alternatively, graphs may be used to introduce data-adaptive distances, such as \emph{diffusion distances}, which are associated to random walks and diffusion processes on graphs \citep{Coifman2005,Nadler2006, Lafon2006, Coifman2008, Singer2008, Rohrdanz2011, Zheng2011, Lederman2015_1, Lederman2015_2, Czaja2016, Li2017efficient}.  Diffusion distances do not depend only on the graph itself, but also on a time parameter $t$ that determines a scale {\em{on the graph}} at which these distances are considered, related to the time of diffusion or random walk. Choosing $\sigma$ in graph-based algorithms, and both $\sigma$ and $t$ in the case of diffusion distances is important in both theory and applications \citep{Szlam2008regularization, Rohrdanz2011, Zheng2011, Murphy2018_Unsupervised}. However their role is well-understood only in certain regimes (e.g. $\sigma,t\rightarrow0^+$) which are of interest in some problems (e.g. manifold learning) but not necessarily in the case of clustering.

We propose the \emph{Learning by Unsupervised Nonlinear Diffusion (LUND)} scheme for clustering, which combines diffusion distances and density estimation to efficiently cluster data generated from a nonparametric model.  This method is applied to the empirical study of high-dimensional hyperspectral images by \citet{Murphy2018_Diffusion, Murphy2018_Iterative, Murphy2018_Unsupervised}, where it is shown to enjoy competitive performance with state-of-the-art clustering algorithms on specific data sets. At the same time, we advance the understanding of the relationship between the local ``spatial'' scale parameter $\sigma$ and the diffusion time parameter $t$ in the context of clustering, demonstrating how the role of $t$ can be exploited to successfully cluster data sets for which $K$-means, spectral clustering, or density-based clustering methods fail, and providing quantitative bounds and guarantees on the performance of the proposed clustering algorithm for data that may be highly nonlinear and of variable density. We moreover provide sufficient conditions under which LUND correctly estimates the number of clusters $K$.


\subsection{Major Contributions and Outline}

This article makes two major contributions.  First, \emph{explicit estimates on diffusion distances for nonparametric clustered data} are proved: we obtain lower bounds for the diffusion distance (see Definition \ref{defn:DD}, or \citep{Coifman2005}) between clusters, and upper bounds on the diffusion distance within clusters, as a function of the time parameter $t$ and suitable properties of the clusters.  Together, these bounds yield a mesoscopic -- not too small, not too large -- diffusion time-scale at which diffusion distances separate clusters clearly and cohere points in the same cluster.  These results, among other things, show how the role of the time parameter, which controls the scale ``on the data'' of the diffusion distances, is very different from the commonly-used scaling parameter $\sigma$ in the construction of the underlying graph, which is a local spatial scale measured in the ambient space. Relationships between $t$, $\sigma$ are well-understood in the asymptotic case of $n\rightarrow+\infty$, $\sigma\rightarrow0^+$ (at an appropriate rate with $n$ \citep{Coifman2005,Lafon2006,VonLuxburg2007}) and $t\rightarrow0^+$ (essentially Varadhan's lemma applied to diffusions on a manifold; see \citep{Den2008_Large}, \citep{Jones2008manifold}, and references therein). These relationships imply that the choice of $t$ is essentially irrelevant, since in these limits of diffusion distances are essentially geodesic distances. However the clustering phenomena we are interested in are far from this regime, and we show that the interplay between $t$, $\sigma$ and $n$ becomes crucial.

Second, \emph{the LUND clustering scheme} is proposed and shown to enjoy performance guarantees for clustering on a broad class of non-parametric mixture models.  We prove sufficient conditions for LUND  to correctly determine the number of clusters in the data and to have low clustering error.  From the computational perspective, we present an efficient algorithm implementing LUND, which scales essentially linearly in the number of points $n$, and in the ambient dimension $D$, for intrinsically low-dimensional data.  We verify the properties of the LUND scheme and algorithm on synthetic data, studying the relationships between the different parameters in LUND, in particular between $\sigma$ and $t$, and comparing with popular and related clustering algorithms, including \emph{spectral clustering} and \emph{fast search and find of density peaks clustering (FSFDPC)} \citep{Rodriguez2014}, illustrating weaknesses of these methods and corresponding advantages of LUND.  Our experiments illustrate how LUND combines benefits of spectral clustering and FSFDPC, allowing it to learn non-linear structure in data while also being guided by regions of high density.  

The outline of the article is as follows.  Background is presented in Section \ref{sec:Background}.  In Section \ref{sec:DataModel}, motivational datasets and a summary of the theoretical results are presented and discussed.  Theoretical comparisons with spectral clustering and density methods are also made in Section \ref{sec:DataModel}.  Estimates on diffusion distances are proved in Section \ref{sec:DiffusionProcesses}.  Performance guarantees for the LUND algorithm are proved in Section \ref{sec:PerformanceGuarantees}.  Numerical experiments and computational complexity are discussed in Section \ref{sec:NumericalExperiments}.  Conclusions and future research directions are given in Section \ref{sec:Conclusions}.


\section{Background}\label{sec:Background}


\subsection{Background on Clustering}

\emph{Clustering} is the process of determining groupings within data and assigning labels to data points according to these groupings---without supervision.  Given the wide variety of data of interest to scientific practitioners, many approaches to clustering have been developed, whose performance is often wildly variable and data-dependent.  Mathematical assumptions are placed on the data to prove performance guarantees.  


\subsubsection{$\numclust$-Means}

A classical and popular clustering algorithm is \emph{$\numclust$-means} \citep{Steinhaus1957, Friedman2001} and its variants \citep{Ostrovsky2006, Arthur2007, Park2009}, which is often used in conjunction with feature extraction methods.  In $K$-means the data is partitioned into $\numclust$ (a user-specified parameter) groups, where the partition $\{C_k\}_{k=1}^\numclust$ is chosen to minimize within-cluster dissimilarity: $C^{*}=\argmin_{\{C_{k}\}_{k=1}^{\numclust}} \sum_{k=1}^{\numclust}\sum_{x\in C_{k}}\|x-\bar{x}_{k}\|_{2}^{2},$
where $\bar{x}_{k}$ is the mean of the $k^{th}$ cluster (for a given partition, it is the minimizer of the least squares cost in the inner sum).  
While very popular in practice, $\numclust$-means and its variants are known to perform poorly for datasets that are not the union of well-separated, near-spherical clusters, and are often sensitive to outliers.  


\subsubsection{Spectral Methods}

The family of clustering methods known as \emph{spectral methods or spectral clustering} compute features that reveal the structure of data that may deviate from the spherical, Gaussian shapes ideal for $\numclust$-means, and in particular may be nonlinear or elongated in shape.  This is done by building local connectivity graphs on the data that encode pairwise similarities between points, then computing a spectral decomposition of adjacency or random walk or Laplacian operators defined on this graph. Focusing on the graph Laplacian $L$ (the other operators are related), one uses the eigenvectors of $L$ as global features input to $\numclust$-means, enabling clustering of nonlinear data that $\numclust$-means alone would fail to cluster accurately.  

More precisely, let $X=\{x_{i}\}_{i=1}^{n}\subset\mathbb{R}^{D}$ be a set of points to cluster.  Let $\G$ be a graph with vertices corresponding to points of $X$ and edges stored in an $n\times n$ symmetric weight matrix $W$.  Often one chooses $W_{ij}=\K(x_{i},x_{j})$ for some (symmetric, often radial and rapidly decaying) nonnegative \emph{kernel} $\K:\mathbb{R}^{D}\times\mathbb{R}^{D}\rightarrow\mathbb{R}$, such as $\K(x_i,x_j)=e^{-||x_i-x_j||^2/\sigma^2}$ for some choice of scaling parameter $\sigma>0$. The graph $\G$ may be fully connected, or it may be a nearest neighbors graph with respect to some metric.  Let $\Deg$ be the diagonal matrix $D_{ii}:=\sum_{j=1}^{n}W_{ij}$.  The \emph{graph Laplacian} is constructed as $L=\Deg-W$.  One then normalizes $L$ to acquire either the \emph{random walk Laplacian} $L_{\text{RW}}=\Deg^{-1}L=\I-\Deg^{-1}W$ or the \emph{symmetric normalized Laplacian} $\Lsym=\Deg^{-\frac{1}{2}}L\Deg^{-\frac{1}{2}}=\I-\Deg^{-\frac{1}{2}}W\Deg^{-\frac{1}{2}}$.  We focus on $\Lsym$ in what follows.  It can be shown that $\Lsym$ has real eigenvalues $0=\lambda_{1}\le \dots \le \lambda_{n}\le 2$ and corresponding eigenvectors $\{\phi_{i}\}_{i=1}^{n}$.  The original data $X$ can be clustered by clustering the embedded data $x_{i}\mapsto (\phi_{1}(x_{i}),\phi_{2}(x_{i}), \dots, \phi_{M}(x_{i}))$ for an appropriate choice of $M\le n$.  In this step typically $\numclust$-means is used, though Gaussian mixture models may (and perhaps should) be used, as they enjoy, unlike $K$-means, a suitably-defined statistical consistency guarantee in the infinite sample limit \citep{Athreya2017_Statistical}.  If clusters in the original data are sufficiently far apart and points within a cluster sufficiently nearby in a suitable sense, spectral clustering with an appropriate kernel can drastically improve over $\numclust$-means \citep{Ng2002}.

It is well-known that spectral clustering relaxes a graph-cut problem.  For a collection of subsets $X_{1},\dots,X_{\numclust}\subset X$, the corresponding \emph{normalized cut} is $\Ncut(X_{1},\dots,X_{\numclust})=\sum_{k=1}^{\numclust}\cut(X_{k},X_{k}^{c})/\vol(X_{k}),$where $\cut(A,B)=\sum_{x_{i}\in A, x_{j}\in B}W_{ij}, \ \vol(A)=\sum_{x_{i}\in A}\sum_{j=1}^{n}W_{ij}.$  One can partition $\G$ by finding a partition minimizing the $\Ncut$ quantity, yielding clusters that are simultaneously separated and balanced \citep{Shi2000}.  However, computing the minimal $\Ncut$ is NP-hard \citep{Wagner1993}.  To relax this problem, one notes that $\min_{X_{1},\dots,X_{\numclust}}\Ncut(X_{1},\dots,X_{\numclust}) =\min_{X_{1},\dots,X_{\numclust}}\Tr(H^{T}LH) \text{ s.t. } H^{T}DH=I,$ where $D,L$ are as above and $H_{ij}=\vol(X_{i})^{-\frac{1}{2}}$ for $x_{i} \in X_{j}$, and $H_{ij}=0$ for $x_{i}\notin X_{j}$.  This formulation may be relaxed to $\min_{U}\Tr(U^{T}D^{-\frac{1}{2}}LD^{\frac{1}{2}}U) \text{ s.t. } U^{T}U=I$, which has solution given by the matrix $U$ consisting of the first $\numclust$ eigenvectors of $\Lsym$.  Hence, one can approximate the NP-hard problem of minimizing $\Ncut$ with an $O(\numclust n^{2})$ time spectral decomposition problem.

Spectral clustering methods enjoy strong empirical performance on a variety of data sets.  However, they are sensitive to clusters of different sizes and densities, and also to parameters in the construction of the underlying graph. There has been significant work on performance guarantees for spectral clustering, which we discuss in Section \ref{subsec:Comparisons}.

Instead of using the eigenfunctions of the Laplacian, it is equivalent to analyze the corresponding random walk, represented by $\Pm=D^{-1}W$.  The graph cut problem is related to transition probabilities between the partition elements by the following result, which is a discrete counterpart to the classical results in the continuous setting involving Brownian motion (see e.g. \citep{Banuelos1999_Hot} and references therein):

\begin{thm}\label{thm:Meila}(\citep{Meila2001})  Let $\Pm$ be an irreducible, aperiodic Markov chain.  Let $\G$ be the graph with nodes corresponding to the states of $\Pm$, and edge weights between the $i^{th}$ and $j^{th}$ node given by $P_{ij}$.  Let $\boldpi_{0}$ be an initial distribution for the Markov chain, and for $A,B$ disjoint subsets of the state space, let $P(A|B)=P(\boldpi_{0}\Pm \in A\ | \ \boldpi_{0}\in B).$  Then $\Ncut(A,A^{c})=P(A^{c} |  A)+P(A| A^{c})$, where $A^{c}$ is the complement of $A$ in the state space.
\end{thm}

Thus, the normalized graph cut generated by a subset $A$ is essentially the same as the probability of transitioning between the sets $A$ and $A^{c}$ in \emph{one} time step, according to the transition matrix $\Pm$.  A crucial aspect of the analysis proposed in this article is to study the behavior across \emph{many} time steps, which makes the proposed method quite different from spectral clustering. 

\begin{figure}[!htb] 
\centering
\begin{subfigure}{.32\textwidth}
\includegraphics[width=\textwidth]{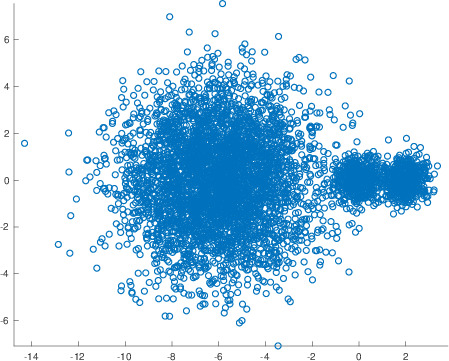}
\subcaption{Data to Cluster \citep{Nadler2007_Fundamental}}
\end{subfigure}
\begin{subfigure}{.32\textwidth}
\includegraphics[width=\textwidth]{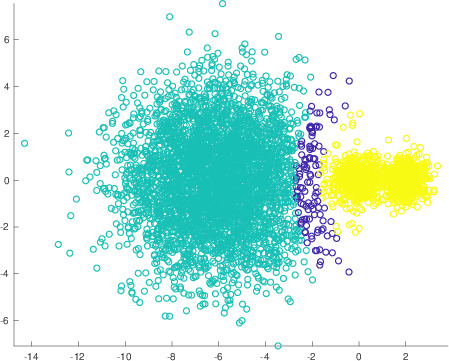}
\subcaption{Spectral clustering \citep{Shi2000}}
\end{subfigure}
\begin{subfigure}{.32\textwidth}
\includegraphics[width=\textwidth]{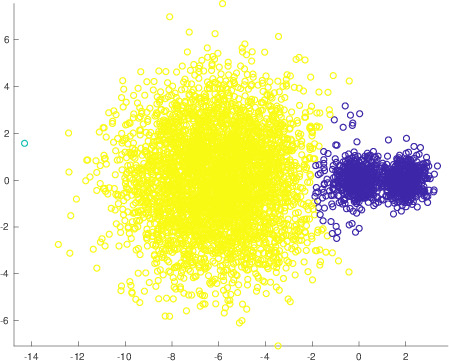}
\subcaption{Spectral clustering \citep{Ng2002}}
\end{subfigure}
\begin{subfigure}{.19\textwidth}
\includegraphics[width=\textwidth]{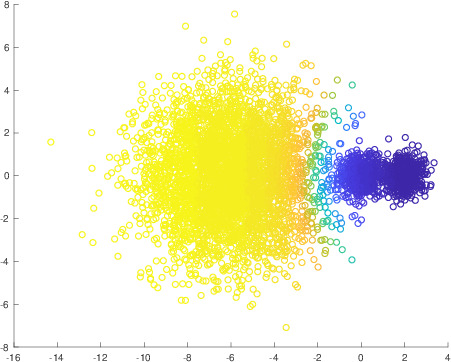}
\subcaption{Eigenvector 2}
\end{subfigure}
\begin{subfigure}{.19\textwidth}
\includegraphics[width=\textwidth]{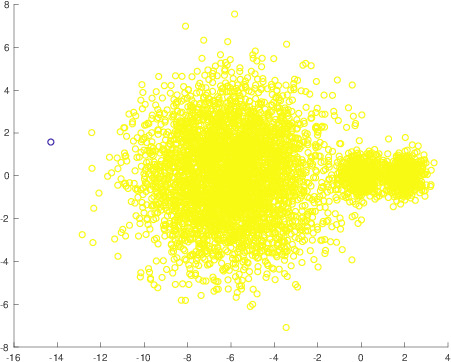}
\subcaption{Eigenvector 3}
\end{subfigure}
\begin{subfigure}{.19\textwidth}
\includegraphics[width=\textwidth]{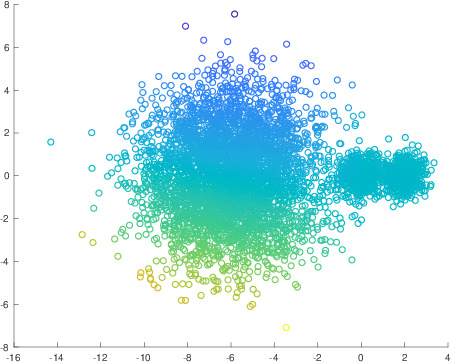}
\subcaption{Eigenvector 4 }
\end{subfigure}
\begin{subfigure}{.19\textwidth}
\includegraphics[width=\textwidth]{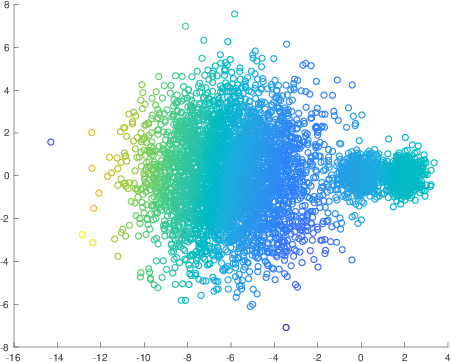}
\subcaption{Eigenvector 5}
\end{subfigure}
\begin{subfigure}{.19\textwidth}
\includegraphics[width=\textwidth]{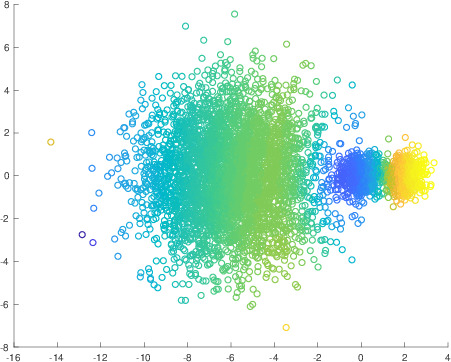}
\subcaption{Eigenvector 6}
\end{subfigure}
\caption{\label{fig:NadlerGaussians}In (a), three Gaussians of essentially the same density are shown.  Results of spectral clustering are shown in (b) \citep{Shi2000} and (c) \citep{Ng2002}.  In (d) - (h), the first five non-trivial eigenvectors are shown.  As noted by \citet{Nadler2007_Fundamental}, the underlying density for this data yields a Fokker-Planck operator whose low-energy eigenfunctions cannot distinguish between the two smaller clusters, thus preventing spectral clustering from succeeding: higher energy eigenfunctions are required.  For this example, the sixth non-trivial eigenvector localizes sufficiently on the small clusters to allow for correct determination of the cluster structure; this eigenvector is not used in traditional spectral clustering algorithms.}
\end{figure}

Weaknesses of spectral clustering were scrutinized by \citet{Nadler2007_Fundamental}.  Their approach starts from the observation in \citep{Coifman2005} that a random walk matrix $\Pm$---defined on $\{x_{i}\}_{i=1}^{n}$ sampled from an underlying density $p(x)$ proportional to $e^{-U(x)/2}$ for some potential function $U(x)$---converges under a suitable scaling as $n\rightarrow\infty$ to the stochastic differential equation (SDE) $\dot{x}(t)=-\nabla U(x) +\sqrt{2}\dot{w}(t),$ where $w$ is the Brownian motion.  Moreover, the top eigenvectors of $\Pm$ converge to the eigenfunctions of the Fokker-Planck operator $\mathcal{L}\psi(x)=\Delta \psi-\nabla\psi\cdot \nabla U=-\mu\psi(x).$  The characteristic time scales of the SDE determine the structure of the leading eigenfunctions of $\mathcal{L}$ \citep{Gardiner2009_Stochastic}: they correspond to the time scales of the slowest transitions between different clusters and the equilibrium times within clusters. The relationships between these quantities determine which eigenfunctions of $\mathcal{L}$ (or $\Pm$) reveal the cluster structure in the data.  Related connections between normalized cuts and exit times from certain clusters are analyzed in \citep{Gavish2013_Normalized}. In \citep{Nadler2007_Fundamental} several examples are presented, including that of three Gaussian clusters of different sizes and densities shown in Figure \ref{fig:NadlerGaussians}.  These data cannot be clustered by spectral clustering using either the second \citep{Shi2000} or the second, third, and fourth eigenfunction \citep{Ng2002}.


\subsubsection{Density and Mode-Based Methods}

Density and mode-based clustering methods detect regions of high-density and low-density to determine clusters.  The \emph{DBSCAN} \citep{Ester1996} and \emph{DBCLASD} \citep{Xu1998} algorithms assign to the same cluster points that are close and have many near neighbors, and flag as outliers points that lie alone in low-density regions.  The \emph{mean-shift} algorithm \citep{Fukunaga1975, Comaniciu2002} pushes points towards regions of high-density, and associate clusters with these high-density points.  In this sense, mean-shift clustering computes \emph{modes} in the data and assigns points to their nearest mode.  Both DBSCAN and mean-shift clustering suffer from a lack of robustness to outliers and depend strongly on parameter choices.  

The recent and popular \emph{fast search and find of density peaks clustering algorithm} (FSFDPC) \citep{Rodriguez2014} proposes to address these weaknesses.  This method characterizes class modes as points of high-density that are far in Euclidean distance from other points of high-density.  This algorithm has seen abundant applications to scientific problems \citep{Spitzer2015, Wiwie2015comparing, Sun2015, Rossant2016, Wang2016semantic, Jia2016}; indeed, this algorithm would cluster the example in Figure \ref{fig:NadlerGaussians} correctly.  However, little mathematical justification for this approach has been given.  We show that the standard FSFDPC method, while very popular in scientific applications, does not correctly cluster the data unless strong assumptions on the data are satisfied.  The main reason is that Euclidean distances are used to find modes, which is inappropriate for data drawn from mixtures of multimodal distributions or distributions nearly supported on nonlinear sets, see for example Figure \ref{fig:NonlinearLUNDversusFSFDPC}.


\subsection{Background on Diffusion Distances}
\label{sec:BackgroundDiffusionDistance}
One of the primary tools in the proposed clustering algorithm is \emph{diffusion distances}, a class of data-dependent distances computed by constructing Markov processes on data \citep{Coifman2005, Coifman2006} that capture its intrinsic structure.  We consider diffusion on the point cloud $X=\{x_{i}\}_{i=1}^{n}\subset\mathbb{R}^{D}$ via a Markov chain \citep{Levin2009} with state space $X$.  Let $\Pm$ be the corresponding $n\times n$ transition matrix.  The following shall be referred to as the \emph{usual assumptions} on $\Pm$: $\Pm$ is reversible, irreducible, aperiodic, and therefore ergodic.  A common construction for $\Pm$, and the one we consider in the algorithmic sections of this article, is to first compute a \emph{weight matrix} $W$, where $W_{ij}=e^{-d(x_{i},x_{j})^{2}/\sigma^{2}}, \ i\neq j$ for some appropriate scale parameter $\sigma\in (0,\infty)$, and $d:\mathbb{R}^{D}\times\mathbb{R}^{D}\rightarrow\mathbb{R}$ is a metric, typically the $\ell^{2}$ norm.  The parameter $\sigma$ encodes the interaction radius of each point: $\sigma$ large allows for long-range interactions between points that are far apart in $\ell^{2}$ norm, while $\sigma$ small allows only for short-range interactions between points that are close in $\ell^{2}$ norm.  $\Pm$ is constructed from $W$ by row-normalizing $W$ so that\footnote{Note that with some abuse of notation we denote the entries of $\Pm$ by $P_{ij}$, reserving the notation $\Pm_{ij}$  for block submatrices of $\Pm$ that will be introduced and used later.} $\sum_{j=1}^{n}P_{ij}=1, \ \forall i=1,\dots,n$.
A unique stationary distribution $\boldpi\in\mathbb{R}^{1\times n}$ satisfying $\boldpi \Pm=\boldpi$ is guaranteed to exist since $\Pm$ is ergodic.  In fact, $\boldpi_{i}=d(x_{i})/\sum_{x\in X}d(x),$ where $d(x_{i})=\sum\limits_{x_{j}\in X}P_{ij}$ is the \emph{degree} of $x_{i}$.

Diffusion processes on graphs lead to a data-dependent notion of distance, known as \emph{diffusion distance} \citep{Coifman2005, Coifman2006}.  While the focus of the construction is on diffusion distances and the diffusion process itself, we mention that \emph{diffusion maps} provide a way of computing and visualizing diffusion distances in Euclidean space, and may be understood as a type of non-linear dimension reduction, in which data in a high number of dimensions may be embedded in a low-dimensional space by a nonlinear coordinate transformation.  In this regard, diffusion maps are related to nonlinear dimension reduction techniques such as Isomap \citep{Tenenbaum2000}, Laplacian eigenmaps \citep{Belkin2003}, and local linear embedding \citep{Roweis2000}, among many others. We focus on the (random walk) process itself.

\begin{defn}\label{defn:DD}
Let $X=\{x_{i}\}_{i=1}^{n}\subset\mathbb{R}^{D}$ and let $\Pm$ be a Markov process on $X$ satisfying the usual assumptions and with stationary distribution $\boldpi$.  Let $\boldpi_{0}$ be a probability distribution on $X$.  For points $x_{i},x_{j}\in X$, let $p_{t}(x_{i},x_{j})=(P^{t})_{ij},$ for some $t\in[0,\infty)$.  The \emph{diffusion distance at time $t$} between $x,y\in X$ is defined, for $\nu=\boldpi_{0}/\boldpi$, by 
$$ 
D_{t}(x,y)=\sqrt{\sum_{u\in X}\left(p_{t}(x,u)-p_{t}(y,u)\right)^{2} \nu(u)}=||p_t(x,\cdot)-p_t(y,\cdot)||_{\ell^2(\nu)}\,.
$$
\end{defn}

For an initial distribution $\boldpi_{0}\in \mathbb{R}^{n}$ on $X$, the vector $\boldpi_{0} \Pm^{t}$ is the probability over states at time $t\ge 0$.  As $t$ increases, this diffusion process on $X$ evolves according to the connections between the points encoded by $\Pm$.  The computation of $D_{t}(x,y)$ involves summing over all paths of length $t$ connecting $x$ to $y$, hence $D_{t}(x,y)$ is small if $x,y$ are strongly connected in the graph according to $\Pm^{t}$, and large if $x,y$ are weakly connected in the graph.  It is known that if the underlying graph is generated from data sampled from a low-dimensional structure, such as a manifold, then diffusion distance parametrizes this low-dimensional structure \citep{Coifman2005, Jones2008manifold, Singer2009detecting, Singer2012vector, Talmon2018latent, Singer2016spectral}.  Indeed, diffusion distances admit a formulation in terms of the eigenfunctions of $\Pm$:

\begin{align}\label{eqn:EigenfunctionsDD}D_{t}^{2}(x,y)=\sum_{\ell=1}^{n}\lambda_{\ell}^{2t}(\psi_{\ell}(x)-\psi_{\ell}(y))^{2}=\lambda_{2}^{2t}\left((\psi_{2}(x)-\psi_{2}(y))^{2}+\sum_{\ell=3}^{\infty}\left(\frac{\lambda_{\ell}}{\lambda_{2}}\right)^{2t}(\psi_{\ell}(x)-\psi_{\ell}(y))^{2}\right)\end{align} 

where $\{(\psi_{\ell},\lambda_{\ell})\}_{\ell=1}^{n}$ is the $\boldpi$-normalized spectral decomposition of $\Pm$, ordered so that $1=\lambda_{1}>\lambda_{2}\ge\lambda_{3}\ge\dots\ge\lambda_{n}>-1$, and noting that $\psi_{1}$ is constant by construction.

Diffusion distances are parametrized by $t$, which measures how long the diffusion process on $\G$ has run when the distances are computed.  Small values of $t$ allow a small amount of diffusion, which may prevent the interesting geometry of $X$ from being discovered, but provide detailed, fine scale information. Large values of $t$ allow the diffusion process to run for so long that the fine geometry may be washed out, leaving only coarse scale information.  We will relate properties of clustered data $X$ to $t$.


\section{Data Model and Overview of Main Results}\label{sec:DataModel}

Among the main results of this article are sufficient conditions for clustering certain discrete data $X\subset \mathbb{R}^{D}$.   The data $X$ is modeled as a realization from a probability distribution 

\begin{align}\label{eqn:MixtureModel}\mu=\sum_{k=1}^{\numclust}w_{k}\mu_{k}, \ w_k\ge0, \ \sum_{k=1}^{\numclust}w_{k}=1,\end{align}

where each $\mu_{k}$ is a probability measure.  Intuitively, our results require \emph{separation and cohesion} conditions on $\{\mu_{k}\}_{k=1}^{\numclust}$.  That is, each $\mu_{k}$ is far from $\mu_{k'}$, $k\neq k'$ and connections are strong (in a suitable sense) within each $\mu_{k}$.
$X = \{x_{i}\}_{i=1}^{n}$ is generated by drawing, for each $i$, one of the $\numclust$ clusters, say $k_i$, according to the multinomial distribution with parameters $(w_1,\dots,w_\numclust)$, and then drawing $x_i$ from $\mu_{k_i}$.  The clusters in the data are defined as the subsets of $X$ whose samples were drawn from a particular $\mu_{k}$, that is, we define the cluster $X_k:=\{x_i\in X: k_i=k\}$.  Given $X$, the goal of clustering is to estimate these $X_{k}$'s accurately, and ideally also determine $\numclust$.  Throughout the theoretical analysis of this article, we will define the accuracy of a set of labels $\{Y_{i}\}_{i=1}^{n}$ learned from an unsupervised algorithm to be $|\{i \ | \ Y_{i}=k_{i}\}|/n,$ i.e. the proportion of points correctly labeled.  

We consider a \emph{nonparametric model} which makes few explicit assumptions on $\mu$.  We will allow $\mu_{k}$ to have nonlinear support and be multimodal (i.e. with multiple high-density regions). These features cause prominent clustering methods to fail, e.g. $\numclust$-means, which requires spherical or extremely separated clusters; spectral clustering, which often fails for highly elongated clusters or clusters of different sizes and densities; or density methods, which are highly sensitive to noise and multi-modality of the distributions.   Two simple, motivating examples are in Figure \ref{fig:MotivatingDatasets}.  They feature variable densities, variable levels of connectivity, both within and across clusters, and (for the second example) nonlinear cluster shapes.  

\begin{figure}[!htb] 
\centering
\begin{subfigure}[t]{.24\textwidth}
\includegraphics[width=\textwidth]{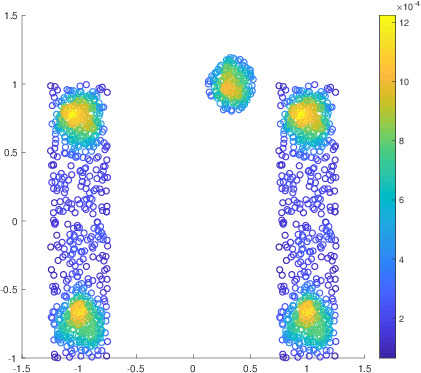}
\subcaption{Bottleneck data}
\end{subfigure}
\begin{subfigure}[t]{.24\textwidth}
\includegraphics[width=\textwidth]{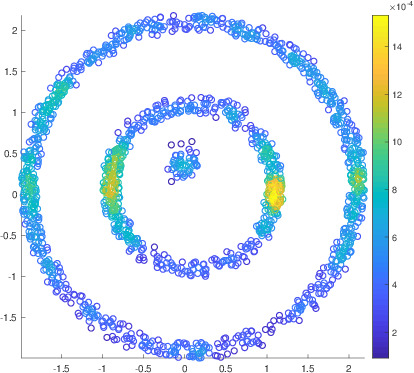}
\subcaption{Nonlinear data}
\end{subfigure}
\begin{subfigure}[t]{.24\textwidth}
\includegraphics[width=\textwidth]{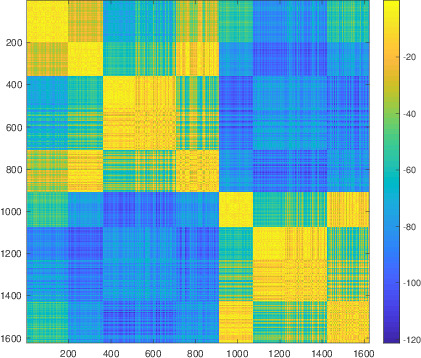}
\subcaption{$\Pm$ for data in (a).}
\end{subfigure}
\begin{subfigure}[t]{.24\textwidth}
\includegraphics[width=\textwidth]{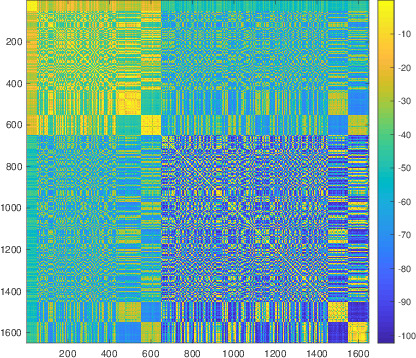}
\subcaption{$\Pm$ for data in (b).}
\end{subfigure}
\caption{\label{fig:MotivatingDatasets}(a), (b) shows two datasets---linear and nonlinear---colored by density.  In (c), (d), we show the corresponding Markov transition matrices $\Pm$, with entry magnitudes shown in $\log_{10}$ scale.  The Markov chains are ergodic, but close to being reducible.  The transition matrices were constructed using the Gaussian kernel as in Section \ref{sec:BackgroundDiffusionDistance}, with the Euclidean distance and $\sigma=.25$}
\end{figure}

The estimates for the behavior of diffusion distances that we derive will then be leveraged to prove that the LUND scheme correctly labels the points, and also estimates the number of clusters, while other clustering schemes fail to cluster these data sets correctly.

\subsection{Summary of Main Results}

Our first result shows that the within-cluster and between-cluster diffusion distances can be controlled, as soon as $\Pm$ is approximately block constant in some sense.  
Define the (worst-case) within-cluster and between-cluster diffusion distances as:
\begin{align}
\label{eqn:InBetween}
\Din=\max_{k}\max_{x,y\in X_{k}}D_{t}(x,y), \quad \Dbtw=\min_{k\neq k'}\min_{x\in X_{k}, y\in X_{k'}}D_{t}(x,y)\,.
\end{align}

The following simplification of Theorem \ref{thm:MainResult} holds:
\begin{thm}
\label{thm:MainResultSummary}
Let $X=\bigcup_{k=1}^{\numclust}X_{k}$ and let $\Pm$ be a corresponding Markov transition matrix on $X$, inducing diffusion distances $\{D_{t}\}_{t\ge0}$. 
Then there exist constants $\{C_{i}\}_{i=1}^{5}\ge0$ such that the following holds:  for any $\epsilon>0$, and for any $t$ satisfying $C_{1}\ln\left(\frac{C_{2}}{\epsilon}\right)<t<C_{3}\epsilon$, we have 
$$
\Din\le C_{4}\epsilon,\quad \Dbtw\ge C_{5}-C_{4}\epsilon.
$$
\end{thm}
The constants depend on a flexible notion of clusterability of $\Pm$.  To get a sense of these constants, let $\Sm^{\infty}$ be an idealized version of $\Pm$, in which all edges from points between clusters are deleted, and redirected back into the cluster, in a sense that will be made precise in Section \ref{sec:DiffusionProcesses}.  If $\Sm^{\infty}$ is constant on each diagonal block corresponding to a cluster, then $C_{1}\approx 0$ and $C_{2}\approx 1$.  If $\Sm^{\infty}$ is a close approximation to $\Pm$, then $C_{3}$ will be large.   If each row of $\Sm^{\infty}-\Pm$ is nearly constant, then $C_{4}\approx 1/\sqrt{n}$.  If $\Sm^{\infty}$ is approximately block constant with blocks of the same size, then $C_{5}\approx{1/\sqrt{n}}$.  If all these conditions hold, the ratio of the bounds $$\frac{\Dbtw}{\Din}\ge\frac{C_{5}-C_{4}\epsilon}{C_{4}\epsilon}=O\left(\frac{1}{\epsilon}\right)$$ suggests strong separation with diffusion distances independently of $n$ when $\epsilon$ is small.  

The LUND scheme is based on the assumption that modes (i.e. high-density points) of the clusters $\{X_{k}\}_{k=1}^{\numclust}$ should have high-density and be far in diffusion distance from other points of high-density, regardless of the shape of the support of the distribution.  Let $p$ be a density estimator on $X$, for example $p(x)=\frac1Z\sum_{y\in NN(x)}e^{-\|x-y\|_{2}^{2}/\sigma^{2}},$ for some choice of $\sigma$ and set of nearest neighbors $NN(x)$, normalized by $Z$ so that $\sum_{x\in X}p(x)=1$.  Given $D_{t}$ defined on $X$, let

\begin{align}\label{eqn:rho_t}\rho_{t}(x)=\begin{cases}\min\limits_{y\in X}\{D_{t}(x,y) \ |\  p(y)\ge p(x)  \}, & x\neq \argmax\limits_{y\in X} p(y), \\ \max\limits_{y\in X}D_{t}(x,y), & x = \argmax\limits_{y\in X} p(y).\end{cases}\end{align}

The function $\rho_{t}$ measures the diffusion distance of a point to its $D_{t}$-nearest neighbor of higher empirical density.   LUND proceeds by estimating one representative mode from each $X_{k}$, then assigning all labels based on these learned modes.  The LUND algorithm is detailed in Algorithm \ref{alg:LUND}.  

\begin{algorithm}[!tb]
	\caption{\label{alg}Learning by Unsupervised Nonlinear Diffusion (LUND) Algorithm}
	\textbf{Input:} $X$ (data), $\sigma$ (scaling parameter), $t$ (time parameter), $\tau$ (threshold)\\
	\textbf{Output:} $Y$ (cluster assignments), $\hat{K}$ (estimated number of clusters)\\
	\begin{algorithmic}[1]
		\STATE Build Markov transition matrix $\Pm$ using scale parameter $\sigma$.  
		\STATE Compute an empirical density estimate $p(x)$ for all $x\in X$.  
		\STATE Compute $\rho_{t}(x)$ for all $x\in X$.
		\STATE Compute $\Dt(x)=\rho_{t}(x)p(x)$ for all $x\in X$.
		\STATE Sort $X$ according to $\Dt(x)$ in descending order as $\{x_{m_{i}}\}_{i=1}^{n}, n=|X|$.  
		\STATE Compute $\hat{K}=\inf_{k}\frac{\Dt(x_{m_{k}})}{\Dt(x_{m_{k+1}})}>\tau$.
		\STATE Assign $Y(x_{m_{i}})=i, \ i=1,\dots, \hat{K}$, and $Y(x_{m_{i}})=0, \ i=\hat{K}+1,\dots,n$.
		\STATE Sort $X$ according to $p(x)$ in decreasing order as $\{x_{\ell_{i}}\}_{i=1}^{n}$.  
		\FOR{$i=1:n$}
		\IF{$Y(x_{\ell_{i}})$=0}
		\STATE $Y(x_{\ell_{i}})=Y(x^{*}), \  x^{*}=\argmin_{y}\{D_{t}(x_{\ell_{i}},y) \ | \ p(y)\ge p(x_{\ell_{i}}) \text{ and } y \text{ is labeled}\}$.
		\ENDIF
		\ENDFOR	

		\end{algorithmic}
		\label{alg:LUND}
\end{algorithm}
  
The LUND algorithm combines density estimation (as captured by $p(x)$) with diffusion geometry (as captured by $\rho_{t}(x)$).  The crucial parameter of LUND is the time parameter, which determines the diffusion distance $D_{t}$ used.  Theorem \ref{thm:MainResultSummary} may be used to show that there is a wide range of $t$ for which applying the proposed LUND algorithm is provably accurate.  The first concern is to understand conditions guaranteeing these modes are estimated accurately, the second that all other points are consequently labeled correctly.  Let $\M=\{p(x) \ | \ \exists k \text{ such that }x = \argmax_{y\in X_{k}} p(y)\}$ be the set of cluster density maxima.  Define $\Dt(x)=p(x)\rho_{t}(x)$.   The following result summarizes Corollaries \ref{cor:NumClusters} and \ref{cor:PerformanceGuarantees}.

\begin{thm}\label{thm:PerformanceGuaranteesSummary}  Suppose $X=\bigcup_{k=1}^{\numclust}X_{k}$ as above. LUND labels all points accurately, and correctly estimates $\numclust$, provided that
\begin{equation}
\frac{\Din}{\Dbtw}<\frac{\min(\M)}{\max(\M)}\,.
\label{e:DratioMratio}
\end{equation}
\end{thm}

Theorem \ref{thm:MainResultSummary} suggests that the condition \eqref{e:DratioMratio} will hold for a wide range of $t$ for a variety of data $X$, so that together with Theorem \ref{thm:PerformanceGuaranteesSummary}, the proposed method correctly labels the data and estimates the number of clusters $\numclust$ correctly.  Note that \eqref{e:DratioMratio} implicitly relates the density of the separate clusters to their geometric properties.  Indeed, if the clusters are well-separated and cohesive enough, then ${\Din}/{\Dbtw}$ is very small, and a large discrepancy in the density of the clusters can be tolerated in inequality \eqref{e:DratioMratio}.  Note that $\Din, \Dbtw, \min(\M),$ and $\max(\M)$ are invariant to increasing $n$, as long as the scale parameter in the kernel used for constructing diffusion distances and the kernel density estimator adjusts according to standard convergence results for graph Laplacians \citep{Belkin2005_Towards,Belkin2007_Convergence,Trillos2016_2,Trillos2018error}. In this sense these quantities are properties of the mixture model.

We note moreover that Theorem \ref{thm:PerformanceGuaranteesSummary} suggests $t$ must be taken in a mesoscopic range, that is, sufficiently far from 0 but also bounded.  Indeed, for $t$ small, $\Din$ is not necessarily small, as the Markov process has not mixed locally yet.  For $t$ large, $\boldpi_{0}\Pm^{t}$ converges to the stationary distribution for all $\boldpi_{0}$, and the ratio $\frac{\Din}{\Dbtw}$ is not necessarily small, since $\Dbtw$ will be small.  In this case, clusters would only be detectable based on density, requiring thresholding, which is susceptible to spurious identification of regions around local density maxima as clusters.


\subsection{Comparisons with Related Clustering Algorithms}\label{subsec:Comparisons}
LUND combines graph-based methods with density-based methods, and it is therefore natural to compare it with spectral clustering and FSFDPC.

\subsubsection{Comparison with Spectral Clustering}

In Theorem \ref{thm:Meila} the graph-cut problem in spectral clustering is related to the probability of transitioning between clusters in \emph{one} time step. LUND uses intermediate time scales to separate clusters, namely the time scale at which the random walk has almost reached the stationary distribution conditioned on not leaving a cluster, and has not yet transitioned (with sizeable probability) to a different cluster.  

Spectral clustering enjoys performance guarantees under a range of model assumptions \citep{Chen2009foundations, Chen2009spectral, Arias2011, AriasChen2011, Vidal2011, Zhang2012, Elhamifar2013, Wang2015_Multi, Soltanolkotabi2014, Arias2017,Little2017Path}.  Under nonparametric assumptions on (\ref{eqn:MixtureModel}) with $\numclust=2$, \citet{Shi2009} show that the principal eigenfunctions and eigenvalues of the associated kernel operator $\mathcal{K}(f)(x)={\int K(x,y)f(x)d\mu(y)}$ are closely approximated by the principal spectra of the kernel operators $\mathcal{K}_{i}(f)(x)={\int K(x,y)f(x)d\mu_{i}(y)}, i=1,2$, possibly mixed up, depending on the spectra of $\mathcal{K}_{1},\mathcal{K}_{2}$ and the weights $w_{1},w_{2}$.  This allows for the number of classes to be estimated accurately in some situations, and for points to be labeled by determining which distribution certain eigenvectors come from.  

The related work of \citet{Schiebinger2015} provides sufficient conditions under the nonparametric model (\ref{eqn:MixtureModel}) for the low-dimensional embedding of spectral clustering to map well-separated, coherent regions in input space to approximately orthogonal regions in the embedding space.  This in turn implies that $\numclust$-means clustering succeeds with high probability, thereby yielding guarantees on the accuracy of spectral clustering.  These results depend on two quantities: with $\mu$ as in (\ref{eqn:MixtureModel}) and $\K$ a kernel, they define separation and cohesion quantities (LUND uses $\Dbtw, \Din$), respectively, as $\mathcal{S}(\mu)=\max_{i\neq j}\mathcal{S}(\mu_{i},\mu_{j}), \GammaMin(\mu)=\min_{i=1,\dots,\numclust}\Gamma(\mu_{i})$, where 

\begin{eqnarray*}
\mathcal{S}(\mu_{i},\mu_{j})=\frac{1}{p(X)}\int_{X}\int_{X}\K(x,y)d\mu_{i}(x)d\mu_{j}(y),\Gamma(\mu_{i})=\inf_{S\subset X}\frac{p(X)}{p(S)p(S^{c})}\int_{S}\int_{S^{c}}\K(x,y)d\mu_{i}(x)d\mu_{i}(y),\end{eqnarray*}
 $p(S)={\int_{S} \int_{X}\K(x,y)d\mu(x)d\mu(y)}.$  A major result of \citet{Schiebinger2015} is that spectral clustering is accurate with high probability depending on a confidence parameter $\beta$ and the number of data samples $n$ if   
 
\begin{align}
\label{eqn:SchiebingerBound}
\frac{\sqrt{\numclust(\mathcal{S}(\mu)+\mathcal{C}(\mu))}}{\min_{i=1,\dots \numclust}w_{i}}+\left(\frac{1}{\sqrt{n}}+\beta\right)\lesssim \GammaMin^{4}(\mu)\,,
\end{align}

where $\mathcal{C}(\mu)$ is a ``coupling parameter" that is not germane to the present discussion.  (\ref{eqn:SchiebingerBound}) holds when the within-cluster coherence $\GammaMin(\mu)$ is large relative to the similarity between clusters $\mathcal{S}(\mu)$.  Fixing the separation $\GammaMin(\mu)$, (\ref{eqn:SchiebingerBound}) is more likely to hold if the clusters are relatively spherical in shape.  For example, in Figure \ref{fig:CompactVersusElongated} we represent two data sets, each consisting of two clusters, with comparable $\mathcal{S}(\mu)$, but substantially different $\GammaMin(\mu)$.  Also note that in the finite sample case when $\frac{1}{\sqrt{n}}$ in \eqref{eqn:SchiebingerBound} is non-negligible, the importance of $\GammaMin$ being not too small increases.  

\begin{figure}[!htb] 
\centering
\begin{subfigure}{.49\textwidth}
\includegraphics[width=\textwidth]{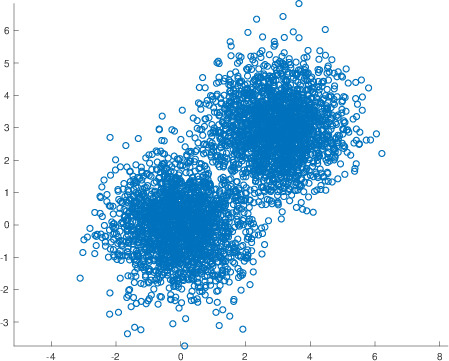}
\subcaption{$\mathcal{S}(\mu)=0.0533, \GammaMin(\mu)\approx 0.9550$.}
\end{subfigure}
\begin{subfigure}{.49\textwidth}
\includegraphics[width=\textwidth]{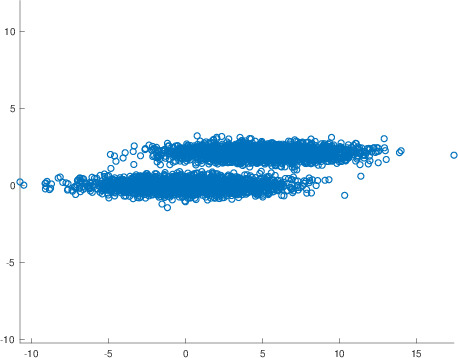}
\subcaption{$\mathcal{S}(\mu)=0.0523,\  \GammaMin(\mu)\approx 0.2560$.}
\end{subfigure}
\caption{\label{fig:CompactVersusElongated}In (a) and (b) two different mixtures of Gaussians are shown. The two mixtures have roughly the same measure of between-cluster distance $\mathcal{S}(\mu)$, but significantly different within-cluster coherence $\GammaMin(\mu)$.  Spectral clustering will enjoy much stronger performance guarantees, according to \citet{Schiebinger2015}, for the data in (a) compared to the data in (b), for a range of relevant choices of the parameter $\sigma$.}
\end{figure}

It is of related interest to compare LUND to spectral clustering by recalling (\ref{eqn:EigenfunctionsDD}).  In the generic case that $\lambda_{2}>\lambda_{3}$, the $(\psi_{2}(x)-\psi_{2}(y))^{2}$ term dominates asymptotically as $t\rightarrow\infty$.  Hence, as $t\rightarrow\infty$, LUND bears resemblance to spectral clustering with the second eigenvector alone \citep{Shi2000}.  On the other extreme, for $t=0$, diffusion distances depend on all eigenvectors equally.  Using the first $K$ or the $2^{nd}$ through $(K+1)^{st}$ eigenvectors $\psi_{l}$ is the basis for many spectral clustering algorithms \citep{Ng2002,Schiebinger2015}, and is comparable to LUND for $t=0$, combined with a truncation of (\ref{eqn:EigenfunctionsDD}).  Note that clustering with the kernel $\K$ alone relates to using all eigenvectors and $t=1$.  By allowing $t$ to be a tunable parameter, LUND interpolates between the extremes of the $\numclust$ principal eigenvectors equally ($t=0$ and cutting off the eigendecomposition after that $\numclust^{th}$ or $(\numclust+1)^{st}$ eigenvector), using the kernel matrix $(t=1)$, and using only the second eigenvector $(t\rightarrow\infty)$.  The results of Section \ref{sec:NumericalExperiments} validate the importance of this flexibility.  

An additional challenge when using spectral clustering is to robustly estimate $\numclust$.  The \emph{eigengap} $\hat{\numclust}=\argmax_{i}\lambda_{i+1}-\lambda_{i}$ is a commonly used heuristic, but is often ineffective when Euclidean distances are used in the case of non-spherical clusters \citep{Arias2011, Little2017Path}.  In contrast, Theorem \ref{thm:PerformanceGuaranteesSummary} suggests LUND can robustly estimate $\numclust$, which is shown empirically for synthetic data in Section \ref{sec:NumericalExperiments}.  

Computationally, LUND and spectral methods are essentially the same, with the bottleneck in complexity being either the spectral decomposition of a dense $n\times n$ matrix ($O(Mn^{2})$ where $M$ is the number of eigenvectors sought), or the computation of nearest neighbors when using a sparse diffusion operator or Laplacian (using an indexing structure for a fast nearest neighbors search, this is $O(C^{d}Dn\log(n))$, where $d$ is the intrinsic dimension of the data).

\subsubsection{Comparison With Local Graph Cutting Algorithms}

The LUND algorithm bears some resemblance to local graph cutting algorithms \citep{Spielman2004_Nearly, Andersen2006_Local, Andersen2008_Local, Andersen2009_Finding, Spielman2013_Local, Spielman2014_Nearly, Yin2017_Local, Fountoulakis2017_Optimization}.  These methods compute a cluster $C$ around a given vertex $v$ such that the conductance of $C$ is high (see Definition \ref{defn:Conductance}), and which can be computed in sublinear time with respect to the total number of vertices in the graph $n$, and in linear time with respect to $|C|$.  In order to avoid an algorithm that scales linearly (or worse) in $n$, global features---such as eigenvectors of a Markov transition matrix or graph Laplacian defined on the data---must be avoided.  The Nibble algorithm \citep{Spielman2013_Local} and related methods \citep{Andersen2009_Finding} compute approximate random walks for points nearby $v$, and truncate steps that take the random walker too far from already explored points.  This accounts for the most important steps a random walker would take, and avoids considering all $n$ vertices of the graph.  In this sense, Nibble and related methods approximate global diffusion with local diffusion in order to compute a local cluster around a prioritized vertex $v$, while LUND uses diffusion to uncover multitemporal structure.

\subsubsection{Comparison with FSFDPC}

The FSFDPC algorithm \citep{Rodriguez2014} learns the modes of distributions in a manner similar to the method proposed in this article, except that the diffusion distance-based quantity $\rho_{t}$ is replaced with a corresponding Euclidean distance-based quantity:

\begin{align*}\rhoEuc(x)=\begin{cases}\min\limits_{y\in X}\{\|x-y\|_{2} \ |\  p(y)\ge p(x)  \}, & x\neq \argmax\limits_{y\in X} p(y), \\ \max\limits_{y\in X}\|x-y\|_{2}, & x = \argmax\limits_{y\in X} p(y).\end{cases}\end{align*}

Points are then iteratively assigned the same label as their nearest Euclidean neighbor of higher density.  This difference is fundamental. Theoretical guarantees for the FDFDPC using Euclidean distances do not accommodate a rich class of distributions and the guarantees proved in this article fail when using $\mathcal{D}^{\text{euc}}(x)=p(x)\rhoEuc(x)$ for computing modes.  This is because for clusters that are multimodal or nonlinear, there is no reason for high-density regions of one cluster to be well-separated in Euclidean distance.  In  Section \ref{sec:NumericalExperiments}, we shall see FSFDPC fails for the motivating data in Section \ref{sec:DataModel}; examples comparing FSFPDC and LUND for real hyperspectral imaging data appear in \citet{Murphy2018_Unsupervised, Murphy2018_Diffusion}.  


\section{Analysis of Diffusion Processes on Data}\label{sec:DiffusionProcesses}

In this section, we derive estimates for diffusion distances.   Let $\Din, \Dbtw$ be as in (\ref{eqn:InBetween}).  The main result of this section is to show there exists an interval $\mathcal{T}\subset[0,\infty]$ so that $\forall t\in \mathcal{T}, \ \Dbtw>\Din,$ that is, for $t\in\mathcal{T}$, within cluster diffusion distance is smaller than between cluster diffusion distance.  Showing that within-cluster distances is small and between-cluster distance is large is essential for any clustering problem.  The benefit of using diffusion distance is its adaptability to the geometry of the data: it is possible that within cluster diffusion distance is less than between cluster diffusion distance, even in the case that the clusters are highly elongated and nonlinear.  This property does not hold when points are compared with Euclidean distances or many other data-independent distances.

\subsection{Near Reducibility of Diffusion Processes}\label{subsec:NearReducibility}

Let $\Pm$ be a Markov chain defined on points $X$ satisfying the usual assumptions with stationary distribution $\boldpi$.  We will sometimes consider $\boldpi$ as a function with domain $X$, other times as a vector with indices $\{1,\dots,|X|\}$. 

For any initial distribution $\boldpi_{0}$, $\lim_{t\rightarrow\infty}\boldpi_{0}\Pm^{t}=\boldpi$ and moreover for any choice of $\nu=\boldpi_{0}/\boldpi,$ $D_{t}(x,y)\rightarrow0$ uniformly as $t\rightarrow\infty$.  One can quantify the rate of this convergence by estimating the convergence rate of $\Pm$ to its stationary distribution in a normalized $\ell^{\infty}$ sense.  

\begin{defn}\label{defn:RPD}
For a discrete Markov chain with transition matrix $\Pm$ and stationary distribution $\boldpi$, the \emph{relative pointwise distance} at time $t$ is $\Delta(t)=\max\limits_{i,j\in\{1,\dots,n\}}|P_{ij}^{t}-\boldpi_{j}|/\boldpi_{j}.$
\end{defn}

The rate of decay of $\Delta(t)$ is regulated by the spectrum of $\Pm$ \citep{Jerrum1989, Sinclair1989}.  Indeed, let $1=\lambda_{1}>\lambda_{2}\ge\dots\ge \lambda_{n}>-1$ be the eigenvalues of $\Pm$; note that $\lambda_{2}<1$ follows from $\Pm$ irreducible and $\lambda_{n}>-1$ follows from $\Pm$ aperiodic \citep{Chung1997}.  Let $\lambda_*=\max_{i=2,\dots,n}|\lambda_{i}|=\max(|\lambda_{2}|,|\lambda_{n}|),  \ \ \boldpi_{\text{min}}=\min_{x\in X}\boldpi(x).$

\begin{thm}\citep{Jerrum1989, Sinclair1989}\label{thm:DeltaLambda}
Let $\Pm$ be the transition matrix of a Markov chain on state space $X$ satisfying the usual assumptions.  Then $\Delta(t)\le\lambda_*^{t}/\boldpi_{\emph{min}}.$
\end{thm}

Instead of analyzing $\lambda_*$, the \emph{conductance} of $X$ may be used to bound $\Delta(t)$.  

\begin{defn}\label{defn:Conductance}
Let $\G$ be a weighted graph on $X$ and let $S\subset X$.  The \emph{conductance of $S$} is $\Phi_{X}(S)=\sum\limits_{x_{i}\in S, x_{j}\in S^{c}}\boldpi_{i}P_{ij} / \min\left(\sum_{x_{i}\in S}\boldpi_{i}, \sum_{x_{i}\in S^{c}}\boldpi_{i}\right)$.  The \emph{conductance of $\G$} is $\Phi(\Pm)=\min_{S\subseteq\G}\Phi_{X}(S)$.
\end{defn}

Methods for estimating the conductance of certain graphs include \emph{Poincar\'{e} estimates} \citep{Diaconis1991, Diaconis1993} and the method of \emph{canonical paths} \citep{Jerrum1989, Sinclair1989,  Aldous2002}.  These approaches estimate $\Phi(\Pm)$ by showing that certain simple paths may be used as a substitute for a generic path in the graphs.  
The conductance is related to $\lambda_{2}$ (see e.g. \citep{Chung1997}):

\begin{thm}(Cheeger's Inequality)  
Let $\G$ be a weighted undirected graph with transition matrix $\Pm$.  Then the second eigenvalue $\lambda_{2}$ of $\Pm$ satisfies  $\Phi(\Pm)^{2}/2\le 1-\lambda_{2}\le 2\Phi(\Pm).$
\end{thm}

Combining Theorem \ref{thm:DeltaLambda} and Cheeger's inequality relates $\Delta(t)$ to $\Phi(\Pm)$.  

\begin{thm}\citep{Jerrum1989, Sinclair1989}\label{thm:DeltaConductance}
Let $\Pm$ be the transition matrix for a Markov chain on $X$ satisfying the usual assumptions.  Suppose $P_{ii}\ge \frac{1}{2}, \ \forall i=1,\dots,n$.  Then $\Delta(t)\le (1-\frac{1}{2}\Phi(\Pm)^{2})^{t}/\boldpi_{\emph{min}}.$
\end{thm}

Note that any Markov chain can be made to satisfy $P_{ii}\ge \frac{1}{2}, \forall i=1,\dots,n$, simply by replacing $\Pm$ with $\frac{1}{2}(\Pm+I)$.  This keeps the same stationary distribution and reduces the conductance by a factor of $\frac{1}{2}$.  
Whether Theorem \ref{thm:DeltaLambda} or \ref{thm:DeltaConductance} is used, the convergence of $\Pm$ towards its stationary distribution is exponential, with rate determined by $\lambda_*$ or $\Phi(\Pm)$, that is, to how close to being reducible the chain is.  Indeed, for $x,y\in X$ and any initial distribution $\boldpi_{0}$,

\begin{align*}
D_{t}(x,y)=&||p_t(x,\cdot)-p_t(y,\cdot)||_{\ell^2(\nu)}
\le ||p_t(x,\cdot)-\boldpi(\cdot)||_{\ell^2(\nu)}+||p_t(y,\cdot)-\boldpi(\cdot)||_{\ell^2(\nu)}\\
\le&2\sqrt{\sum_{u\in X}\max\limits_{z\in X}\frac{|p_{t}(z,u)-\boldpi(u)|^2}{\boldpi(u)^2}\boldpi(u) \boldpi_{0}(u)}\\
\le& 2\Delta(t)\sqrt{\sum_{u\in X}\boldpi(u)\boldpi_{0}(u)}\le 2 \Delta(t)
\le\frac{2(1-\frac{1}{2}\Phi(\Pm)^{2})^{t}}{\boldpi_{\text{min}}}\,.
\end{align*}

Thus, as $t\rightarrow\infty$, $D_{t}\rightarrow 0$ uniformly at an exponential rate depending on the conductance of the underlying graph; a similar result holds for $\lambda_*$ in place of $\Phi(\Pm)$.  This gives a global estimate on the diffusion distance in terms of $\lambda_*$ and $\Phi(\Pm)$.  Note that a similar conclusion holds by analyzing (\ref{eqn:EigenfunctionsDD}), recalling that $\psi_{1}$ is constant and $\lambda_{2}=\max_{i\neq 1}|\lambda_{i}|=\lambda_*$.

Unfortunately, a global mixing time may be too coarse for unsupervised learning.  To obtain the desired separation of $\Din, \Dbtw$, we need to study not the global mixing time, but the \emph{mesoscopic} mixing times, corresponding to the time it takes for convergence of points in each cluster towards their mesoscopic equilibria, before reaching the  global equilibrium.  For this purpose, we use results from the theory of \emph{nearly reducible Markov processes} \citep{Simon1961, Meyer1989}. 
Suppose the matrix $\Pm$ is irreducible; write $\Pm$, possibly after a permutation of the indices of the points, in block decomposition as

\begin{equation}
\Pm=\begin{bmatrix}\label{eqn:Ppartition}
    \Pm_{11}  & \Pm_{12}  & \dots & \Pm_{1m} \\
   \Pm_{21}  & \Pm_{22}  & \dots & \Pm_{2m} \\
   \vdots & \vdots & \ddots & \vdots \\
   \Pm_{m1}  & \Pm_{m2}  & \dots & \Pm_{m m}
\end{bmatrix},
\end{equation}

where each $\Pm_{ii}$ is square and $m\le n$.  Let $I_{i}$ be the indices of the points corresponding to $\Pm_{ii}$.  Recall that if the graph corresponding to $\Pm$ is disconnected, then $\Pm$ is a \emph{reducible} Markov chain; this corresponds to a block diagonal matrix in which $\|\Pm_{ij}\|_{\infty}=0$, $i\neq j$, for some $i$ (recall that $\|\textbf{A}\|_{\infty}=\max_{i}\sum_{j}|a_{ij}|$ is the maximal row sum of $\textbf{A}=(a_{ij})$).  Consider instead that $\|\Pm_{ij}\|_{\infty}$ is small  but nonzero for $i\neq j$, that is, most of the interactions for points in $I_{i}$ are contained within $\Pm_{ii}$.  This suggests diffusion on the the blocks $\Pm_{ii}$ have dynamics that converge to their own, mesoscopic equilibria before the entire chain converges to a global equilibrium, depending on the weakness of connection between blocks.  Interpreting the support sets $I_{i}$ as corresponding to the clusters of $X$, this suggests there will be a time range for which points within each cluster are close in diffusion distance but far in diffusion distance from points in other clusters; such a state corresponds to a mesoscopic equilibrium.  To make this precise, consider the notion of \emph{stochastic complement}.

\begin{defn}  Let $\Pm$ be an $n\times n$ irreducible Markov matrix partitioned into square block matrices as in (\ref{eqn:Ppartition}).  For a given index $i\in\{1,\dots,m\}$, let $\Pm_{i}$ denote the principal block submatrix generated by deleting the $i^{\text{th}}$ row and $i^{\text{th}}$ column of blocks from (\ref{eqn:Ppartition}), and let $\Pm_{*i}=\begin{bmatrix}\Pm_{1i}   \Pm_{2i} \dots \Pm_{i-1,i}  \Pm_{i+1,i} \dots  \Pm_{mi}\end{bmatrix}^{T}$ and $\Pm_{i*}=\begin{bmatrix} \Pm_{i1} \ \Pm_{i2}\  \dots \ \Pm_{i,i-1} \ \Pm_{i,i+1} \ \dots \ \Pm_{im} \end{bmatrix}$.  The \emph{stochastic complement of $\Pm_{ii}$} is the matrix $\Sm_{ii}=\Pm_{ii}+\Pm_{i*}(\I-\Pm_{i})^{-1}\Pm_{*i}.$
\end{defn}

One can interpret the stochastic complement $\Sm_{ii}$ as the transition matrix for a reduced Markov chain obtained from the original chain, but in which transitions into or out of $I_{i}$ are masked.  More precisely, in the reduced chain $\Sm_{ii}$, a transition is either direct in $\Pm_{ii}$ or indirect by moving first through points outside of $\Pm_{ii}$, then back into $\Pm_{ii}$ at some future time.  Indeed, the term $\Pm_{i*}(\I-\Pm_{i})^{-1}\Pm_{*i}$ in the definition of $\Sm_{ii}$ accounts for leaving $I_{i}$ (the factor $\Pm_{i*}$), traveling for some time in $I_{i}^{c}$ (the factor $(\I-\Pm_{i})^{-1}$), then re-entering $I_{i}$ (the factor $\Pm_{*i}$).  Note that the factor $(\I-\Pm_{i})^{-1}$ may be expanded in Neumann sum as $(\I-\Pm_{i})^{-1}=\sum_{t=0}^{\infty}\Pm_{i}^{t},$ clearly showing that it accounts for exiting from $I_i$ and returning to it after an arbitrary number of steps outside of it.

The notion of stochastic complement quantifies the interplay between the mesoscopic and global equilibria of $\Pm$.  We say $\Pm$ is \emph{primitive} if it is non-negative, irreducible and aperiodic.  The following theorem indicates how $\Pm$ may be analyzed when it is derived from cluster data $\{X_{k}\}_{k=1}^{\numclust}$ sampled according to (\ref{eqn:MixtureModel}); a proof appears in the Appendix for completeness.

\begin{thm}\label{thm:NearReducibility}\citep{Meyer1989}  Let $\Pm$ be an $n\times n$ irreducible row-stochastic matrix partitioned into $K^{2}$ square block matrices, and let $\Sm$ be the completely reducible row-stochastic matrix consisting of the stochastic complements of the diagonal blocks of $\Pm$:

\[ 
\Pm=\begin{bmatrix}
    \Pm_{11}  & \Pm_{12}  & \dots & \Pm_{1\numclust} \\
   \Pm_{21}  & \Pm_{22}  & \dots & \Pm_{2\numclust} \\
   \vdots & \vdots & \ddots & \vdots \\

   \Pm_{\numclust 1}  & \Pm_{\numclust 2}  & \dots & \Pm_{\numclust \numclust }
\end{bmatrix}, \
\
\Sm=\begin{bmatrix}
    \Sm_{11}  & \0  & \dots & \0 \\
   \0  & \Sm_{22}  & \dots & \0 \\
   \vdots & \vdots & \ddots & \vdots \\

   \0  & \0  & \dots & \Sm_{\numclust \numclust}
\end{bmatrix}.\]Suppose each $\Sm_{ii}$ is primitive, so that the eigenvalues of $\Sm$ satisfy $\lambda_{1}=\lambda_{2}=\dots=\lambda_{\numclust}=1>\lambda_{\numclust+1}\ge \lambda_{\numclust+2}\ge \dots>-1.$  Let $\Z$ diagonalize $\Sm$, and let \[ 
\Sm^{\infty}=\lim_{t\rightarrow\infty}\Sm^{t}=\begin{bmatrix}
    \bold{1} \boldpi^{1}  & \0  & \dots & \0 \\
   \0  & \bold{1}\boldpi^{2}  & \dots & \0 \\
   \vdots & \vdots & \ddots & \vdots \\

   \0  & \0  & \dots & \bold{1}\boldpi^{K}
\end{bmatrix},\] where $\boldpi^{i}$ is the stationary distribution for $\Sm_{ii}$.  Then $\|\Pm^{t}-\Sm^{\infty}\|_{\infty}\le t\delta + \kappa|\lambda_{\numclust+1}|^{t},$ where $\delta=2\max_{i}\|\Pm_{i*}\|_{\infty}$ and $\kappa=\|\Z\|_{\infty}\|\Z^{-1}\|_{\infty}$.  
Moreover, for any initial distribution $\boldpi_{0}$ and $\s=\lim_{t\rightarrow\infty}\bold\boldpi_{0}\Sm^{t}=\boldpi_0\Sm^\infty$, $\|\boldpi_{0}\Pm^{t}-\s\|_{1}\le t\delta+\kappa|\lambda_{\numclust+1}|^{t}.$
\end{thm}

Note that this result does not require the Markov chain to be reversible, and hence applies to diffusion processes defined on \emph{directed} graphs.  The assumption that $\Sm$ is diagonalizable is not strictly necessary, and similar estimates hold more generally \citep{Meyer1989}.  

The estimate $t\delta+\kappa|\lambda_{\numclust+1}|^{t}$ splits into two terms.  The $t\delta$ term corresponds to $\|\Pm^{t}-\Sm^{t}\|_{\infty}$, which accounts for the approximation of $\Pm^{t}$ by the reducible Markov chain $\Sm^{t}$.  In the context of clustering, this term accounts for the between-cluster connections in $\Pm$.  The term $\kappa |\lambda_{\numclust+1}|^{t}$ corresponds to $\|\Sm^{t}-\Sm^{\infty}\|_{\infty}$, which accounts for propensity of mixing within a cluster.  In the clustering context, this term quantifies the within-cluster distances.  

It follows from Theorem \ref{thm:NearReducibility} that, given $\epsilon$ sufficiently large, there is a range of $t$ for which the dynamics of $\Pm^{t}$ are $\epsilon$-close to the dynamics of the reducible Markov chain $\Sm^{\infty}$:

\begin{cor}\label{cor:CriticalTimeRange}Let $\lambda_{\numclust+1}, \delta, \kappa$ be as in Theorem \ref{thm:NearReducibility}.  Suppose that for some $\epsilon(t)>0$, ${\ln\left(\frac{2\kappa}{\epsilon(t)}\right)}/{\ln\left(\frac{1}{|\lambda_{\numclust+1}|}\right)}<t<\frac{\epsilon(t)}{2\delta}.$  Then $\|\Pm^{t}-\Sm^{\infty}\|_{\infty}<\epsilon(t),$ and for every initial distribution $\boldpi_{0}$, $\|\boldpi_{0}\Pm^{t}-\s\|_{1}<\epsilon(t).$
\end{cor}

In contrast with $t$, the values $\lambda_{\numclust+1}, \delta, \kappa$ may be understood as fixed geometric parameters of the dataset which determine the range of times $t$ at which mesoscopic equilibria are reached.  More precisely, as $n\rightarrow\infty$, $\delta,\kappa$ converge to natural continuous quantities independent of $n$, and \citet{Trillos2018error} proved that as $n\rightarrow\infty$, there is a natural scaling for $\sigma\rightarrow 0$ in which the (random) empirical eigenvalues of $\Pm$ converge in a precise sense to the (deterministic) eigenvalues of a corresponding continuous operator defined on the support of $\mu$.  Thus, the parameters of Theorem \ref{cor:CriticalTimeRange} may be understood as random fluctuations of geometrically intrinsic quantities depending on $\mu$.  In the context of the proposed data model, these quantities may be interpreted as follows:
\begin{itemize}
\item $\lambda_{\numclust+1}$ is the largest eigenvalue of $\Sm$ not equal to 1.  Since $\Sm$ is block diagonal and each $\Sm_{kk}$ is primitive, it follows that $\lambda_{\numclust+1}=\max_{k=1,\dots,\numclust}\lambda_{2}(\Sm_{kk})$. As discussed above, $\{\lambda_{2}(\Sm_{kk})\}_{k=1}^{\numclust}$ is related to the conductance $\Phi(\Sm_{kk})$ and the mixing  time of the random walk restricted to $\Sm_{kk}$.  If the entries of $\Sm_{kk}$ are very close to the entries of $\Pm_{kk}$, then a perturbative argument yields $\lambda_{2}(\Sm_{kk})\approx \lambda_{2}(\Pm_{kk})$.  

\item The quantity $\delta=2\max_{k=1,\dots,\numclust}\|\Pm_{k*}\|_{\infty}$ is controlled by the largest interaction between clusters.  If the separation between the $\{X_{k}\}_{k=1}^{\numclust}$ is large enough, $\delta$ will be small.

\item The quantity $\kappa=\|\Z\|_{\infty}\|\Z^{-1}\|_{\infty}$, with $\Z=\left(\phi_{1}|\dots|\phi_{n}\right)$, is  a measure of the condition number of diagonalizing $\Sm$. If $\Z,\Z^{-1}$ are orthogonal matrices, then each row of $\Z,\Z^{-1}$ have $\ell^{2}$ norm 1, hence  $\kappa\le n$.  We remark that $\kappa$ is bounded independently of $n$ in the case that all the data live on a common manifold, using convergence of heat kernels and low-frequency eigenfunctions together with heat kernel estimates on manifolds.  In the clustering setting, if each cluster is a manifold, similar results would hold in this case, albeit this analysis is a topic of ongoing research.  
\end{itemize}


\subsection{Diffusion Distance Estimates}

Returning to the proposed data model $X=\bigcup_{k=1}^{\numclust}X_{k}\sim\mu$ as per (\ref{eqn:MixtureModel}), let $\Pm$ be a corresponding Markov chain on $X$ satisfying the usual assumptions.  We estimate the dependence of diffusion distances on the parameters $\delta,\lambda_{\numclust+1}, \kappa$ above.  We also introduce a \emph{balance quantity} that quantifies the difference between the $\ell^{1}$ norm (the setting of Theorem \ref{thm:NearReducibility}) and the $\ell^{2}$ norm (the setting of diffusion distances).

\begin{defn}
Let $\Pm, \Sm^{\infty}\in \mathbb{R}^{n\times n}$ be as in Theorem \ref{thm:NearReducibility} and set $p_{t}(x_{i},x_{j})=\Pm^{t}_{ij}, \ \s^{\infty}(x_{i},x_{j})=\Sm_{ij}^{\infty}$.  Define $$\gamma(t)=\max_{x\in X}\left(1-\frac{1}{2}\sum_{u\in X}\left|\frac{|p_{t}(x,u)-\sinf(x,u)|}{\|p_{t}(x,\cdot)-\sinf(x,\cdot)\|_{\ell^{2}}}-\frac{1}{\sqrt{n}}\right|^{2}\right)^{-1}.$$  
\end{defn}

\citet{Botelho2017_Exact} show that for any vector $v\in\mathbb{R}^{n}$, $\|v\|_{\ell^{2}}=\frac{c_{v}}{\sqrt{n}}\|v\|_{\ell^{1}},$ where $$c_{v}=\left(1-\frac{1}{2}\sum_{i=1}^{n}\left|\frac{|v_{i}|}{\|v\|_{\ell^{2}}}-\frac{1}{\sqrt{n}}\right|^{2}\right)^{-1}.$$  In this sense, $\gamma(t)$ measures how the $\ell^{1}$ norm differs from the $\ell^{2}$ norm across all rows of $\Pm^{t}-\Sm^{\infty}$.  In particular, when each row of $\Pm^{t}-\Sm^{\infty}$ is close to uniform, $\gamma(t)$ is close to 1; when some row of $\Pm^{t}-\Sm^{\infty}$ concentrates all its mass around one index, then $\gamma(t)=\sqrt{n}$.  Note that $1\le \gamma(t)\le \sqrt{n}$ for all $t$.

\begin{thm}\label{thm:MainResult}Let $X=\bigcup_{k=1}^{\numclust}X_{k}$ and let $\Pm$ be a corresponding Markov transition matrix on $X$.  Let $\delta, \lambda_{\numclust+1}, \kappa,\Sm^{\infty}$ be as in Theorem \ref{thm:NearReducibility} and let $\s^{\infty}(x_{i},x_{j})=\Sm_{ij}^{\infty}$. Let $D_{t}$ be the diffusion distance associated to $\Pm$ and counting measure $\nu$.  If $t,\epsilon(t)$ satisfy $$\frac{\ln\left(\frac{2\kappa}{\epsilon(t)}\right)}{\ln\left(\frac{1}{\lambda_{\numclust+1}}\right)}<t<\frac{\epsilon(t)}{2\delta}\,,$$ then
\begin{enumerate}[(a)]

\item $\displaystyle\Din\le 2\frac{\epsilon(t)}{\sqrt{n}}\gamma(t)$.

\item $\displaystyle\Dbtw\ge 2\min_{y\in X}\|\sinf(y,\cdot)\|_{\ell^{2}(\nu)}-2\frac{\epsilon(t)}{\sqrt{n}}\gamma(t)$.

\end{enumerate}

\end{thm}

\begin{proof}

By Corollary \ref{cor:CriticalTimeRange}, $\|\Pm^{t}-\Sm^{\infty}\|_{\infty}<\epsilon(t)$, that is, $\max_{x\in X}\sum_{u\in X}|p_{t}(x,u)-\sinf(x,u)|\nu(u)<\epsilon(t).$  To see (a), let $k$ be arbitrary and let $x,y \in X_{k}$.  Then:  \begin{align*}
& ||p_t(x,\cdot)-p_t(y,\cdot)||_{\ell^2(\nu)}\\
\le &  ||p_t(x,\cdot)-\sinf(x,\cdot)||_{\ell^2(\nu)} + ||p_t(y,\cdot)-\sinf(y,\cdot)||_{\ell^2(\nu)} + ||\sinf(y,\cdot)-\sinf(x,\cdot)||_{\ell^2(\nu)}\\
= &  \frac{1}{\sqrt{n}}\left(1-\frac{1}{2}\sum_{u\in X}\left|\frac{|p_{t}(x,u)-\sinf(x,u)|}{\|p_{t}(x,\cdot)-\sinf(x,\cdot)\|_{\ell^{2}(\nu)}}-\frac{1}{\sqrt{n}}\right|^{2}\right)^{-1}||p_t(x,\cdot)-\sinf(x,\cdot)||_{\ell^1(\nu)} \\+& \frac{1}{\sqrt{n}}\left(1-\frac{1}{2}\sum_{u\in X}\left|\frac{|p_{t}(y,u)-\sinf(y,u)|}{\|p_{t}(y,\cdot)-\sinf(y,\cdot)\|_{\ell^{2}(\nu)}}-\frac{1}{\sqrt{n}}\right|^{2}\right)^{-1}||p_t(y,\cdot)-\sinf(y,\cdot)||_{\ell^1(\nu)}\\ +& ||\sinf(y,\cdot)-\sinf(x,\cdot)||_{\ell^2(\nu)}\\
\le &\frac{2\epsilon(t)}{\sqrt{n}}\max_{x\in X}\left(1-\frac{1}{2}\sum_{u\in X}\left|\frac{|p_{t}(x,u)-\sinf(x,u)|}{\|p_{t}(x,\cdot)-\sinf(x,\cdot)\|_{\ell^{2}(\nu)}}-\frac{1}{\sqrt{n}}\right|^{2}\right)^{-1}+||\sinf(y,\cdot)-\sinf(x,\cdot)||_{\ell^2(\nu)}\,,
\end{align*}
where $t$ satisfies $\ln\left(\frac{2\kappa}{\epsilon(t)}\right)/\ln\left(\frac{1}{\lambda_{\numclust+1}}\right)<t<\epsilon(t)/(2\delta).$  The line relating the norm in $\ell^{1}(\nu)$ and $\ell^{2}(\nu)$ follows from Theorem 1 in \citep{Botelho2017_Exact}.  Note that $\Sm^{\infty}$ has constant columns on each cluster, and in particular for $x,y\in X_{k}$, $\sinf(x,u)=\sinf(y,u)=\boldpi^{k}(u)$ for all $u\in X$, so that $||\sinf(y,\cdot)-\sinf(x,\cdot)||_{\ell^2(\nu)}=0$. Statement (a) follows.

To see (b), suppose that $x\in X_{k}, y\in X_{\ell}, \ k\neq \ell$.  Then \begin{align*}&\|p_{t}(x,\cdot)-p_{t}(y,\cdot)\|_{\ell^{2}(\nu)}\\
=&\|p_{t}(x,\cdot)-\sinf(x,\cdot)+\sinf(x,\cdot)-\sinf(y,\cdot)+\sinf(y,\cdot)-p_{t}(y,\cdot)\|_{\ell^{2}(\nu)}\\
\ge& \|\sinf(x,\cdot)-\sinf(y,\cdot)\|_{\ell^{2}(\nu)}-\|p_{t}(x,\cdot)-\sinf(x,\cdot)\|_{\ell^{2}(\nu)}-\|p_{t}(y,\cdot)-\sinf(y,\cdot)\|_{\ell^{2}(\nu)}\\
=&\|\sinf(x,\cdot)-\sinf(y,\cdot)\|_{\ell^{2}(\nu)}\\-&\frac{1}{\sqrt{n}}\left(1-\frac{1}{2}\sum_{u\in X}\left|\frac{|p_{t}(x,u)-\sinf(x,u)|}{\|p_{t}(x,\cdot)-\sinf(x,\cdot)\|_{\ell^{2}(\nu)}}-\frac{1}{\sqrt{n}}\right|^{2}\right)^{-1}||p_t(x,\cdot)-\sinf(x,\cdot)||_{\ell^1(\nu)}\\-&\frac{1}{\sqrt{n}}\left(1-\frac{1}{2}\sum_{u\in X}\left|\frac{|p_{t}(y,u)-\sinf(y,u)|}{\|p_{t}(y,\cdot)-\sinf(y,\cdot)\|_{\ell^{2}(\nu)}}-\frac{1}{\sqrt{n}}\right|^{2}\right)^{-1}||p_t(y,\cdot)-\sinf(y,\cdot)||_{\ell^1(\nu)}\\
\ge&\|\sinf(x,\cdot)-\sinf(y,\cdot)\|_{\ell^{2}(\nu)}-\frac{1}{\sqrt{n}}\left(1-\frac{1}{2}\sum_{u\in X}\left|\frac{|p_{t}(x,u)-\sinf(x,u)|}{\|p_{t}(x,\cdot)-\sinf(x,\cdot)\|_{\ell^{2}(\nu)}}-\frac{1}{\sqrt{n}}\right|^{2}\right)^{-1}\epsilon(t)\\-&\frac{1}{\sqrt{n}}\left(1-\frac{1}{2}\sum_{u\in X}\left|\frac{|p_{t}(y,u)-\sinf(y,u)|}{\|p_{t}(y,\cdot)-\sinf(y,\cdot)\|_{\ell^{2}(\nu)}}-\frac{1}{\sqrt{n}}\right|^{2}\right)^{-1}\epsilon(t)\\
\ge & \|\sinf(x,\cdot)-\sinf(y,\cdot)\|_{\ell^{2}(\nu)}-2\frac{\epsilon(t)}{\sqrt{n}}\max_{z\in X_{\ell}\cup X_{k}}\left(1-\frac{1}{2}\sum_{u\in X}\left|\frac{|p_{t}(z,u)-\sinf(z,u)|}{\|p_{t}(z,\cdot)-\sinf(z,\cdot)\|_{\ell^{2}(\nu)}}-\frac{1}{\sqrt{n}}\right|^{2}\right)^{-1}\\
\ge & 2\min_{w\in X_{k}\cup X_{\ell}}\|\sinf(w,\cdot)\|_{\ell^{2}(\nu)}-2\frac{\epsilon(t)}{\sqrt{n}}\max_{z\in X_{\ell}\cup X_{k}}\left(1-\frac{1}{2}\sum_{u\in X}\left|\frac{|p_{t}(z,u)-\sinf(z,u)|}{\|p_{t}(z,\cdot)-\sinf(z,\cdot)\|_{\ell^{2}(\nu)}}-\frac{1}{\sqrt{n}}\right|^{2}\right)^{-1}\,,
\end{align*}
where in the last step, to lower bound the first term we used that $\sinf(y,\cdot)=\boldpi^{l}(\cdot)$, $\sinf(x,\cdot)=\boldpi^{k}(\cdot)$, and recalled that since $k\neq l$ the supports of $\boldpi^k$ and $\boldpi^l$ are disjoint.
Minimizing this lower bound over all clusters $X_{k}, X_{\ell}$ yields the desired result.  
\end{proof}

Heuristically, if $\epsilon(t)$ is small and the reduced equilibrium distribution $\sinf$ is roughly constant on each cluster, there will be a range of $t$ for which $\Din\ll\Dbtw.$  The notion of $\sinf$ being roughly constant on each cluster is equivalent to nodes in the same cluster having roughly constant degree.  These theoretical estimates are compared to empirical bounds computed numerically in Section \ref{sec:NumericalExperiments}.  

If $\Pm$ is very close to  $\Sm$ in Frobenius norm, then $p_{t}(x,y)$ would be very close to $\sinf(x,y)$ and $\epsilon(t)$ would be close to $0$.  In this case the estimates of Theorem \ref{thm:MainResult} reduce to 
\begin{equation}
\label{eqn:epsilon0}
\begin{aligned}
\Din =0\qquad,\qquad
 \Dbtw\ge 2\min_{y\in X}\|\sinf(y,\cdot)\|_{\ell^{2}(\nu)}.
 \end{aligned}  
\end{equation}
One can define a natural notion of diffusion distance between disjoint clusters in a reducible Markov chain as the sum of the $\ell^{2}$ norms of their respective stationary distribution, which agrees with both the definition of diffusion distances upon taking the limit $t\rightarrow+\infty$ and with the lower bound (b) in Theorem \ref{thm:MainResult} when $\epsilon(t)\rightarrow 0$.  Hence, while the estimates in the proof of Theorem \ref{thm:MainResult} may not be optimal, they are quite natural for $\epsilon(t)\rightarrow 0$.  
 
 Away from the asymptotic regime $\Pm\rightarrow \Sm$, the estimates of Theorem \ref{thm:MainResult} may be further simplified by placing additional assumptions on the data.
\begin{cor}Suppose that $\sinf$ is uniform on each $X_{k}$, and the cardinality of each $X_{k}$ is $\frac{n}{\numclust}$.   Then for any $t,\epsilon(t)$ satisfying ${\ln\left(\frac{2\kappa}{\epsilon(t)}\right)}\big/{\ln\left(\frac{1}{\lambda_{\numclust+1}}\right)}<t<\frac{\epsilon(t)}{2\delta}$, 
$$\Din \le \frac{2}{\sqrt n}\epsilon(t)\gamma(t)\qquad,\qquad
\Dbtw\ge \frac{2}{\sqrt n}\left(\sqrt K-\epsilon(t)\gamma(t)\right).$$
\end{cor}

\begin{proof}If $\Sm^{\infty}$ has constant rows on each cluster (i.e. the stationary distribution on each cluster of the reduced Markov chain is uniform), and the clusters are of constant size $n/K$, then $$2\min_{y\in X}\|\sinf(y,\cdot)\|_{\ell^{2}(\nu)}=2\sqrt{\frac{K}{n}}.$$  The result follows from Theorem \ref{thm:MainResult}.  

\end{proof}

In particular, if $\epsilon(t)\ll \frac{\sqrt{K}}{2\gamma(t)}$, within cluster distances will be small since $\Din\ll \sqrt{\frac{K}{n}}$, and also there will be clear separation between clusters since $\Dbtw=\Omega\left(\sqrt{\frac{K}{n}}\right)$.  Note that when $\Pm^{t}-\Sm^{\infty}$ is balanced, $\gamma(t)$ is $O(1)$ with respect to $n$, so that the assumption on $\epsilon(t)$ is independent of $n$.

\subsection{Example: A Simple Gaussian Mixture Model}\label{subsec:ExampleAnalysis}

The major parameters controlling the estimates in Theorem \ref{thm:MainResult} are $\delta, \lambda_{\numclust+1},$ and $\kappa$.  To illustrate the key quantities of this theorem, we consider the simple example of a mixture of Gaussians $\mu_{G}=\frac{1}{2}\mathcal{N}(\x_{1},\Sigma)+\frac{1}{2}\mathcal{N}(\x_{2},\Sigma)$ in $\mathbb{R}^{2}$ with diagonal isotropic covariance matrix $\Sigma=\frac{1}{10}I$.  Our method depends mainly on the intrinsic geometric constants $\delta$ and $\lambda_{K+1}$, so clustering performance with LUND is robust to small amounts of geometric deformation, for example the action of a bi-Lipschitz map.  We construct the diffusion transition matrix $\Pm$ using the Gaussian kernel with $\sigma=.2$, as described in Section \ref{sec:BackgroundDiffusionDistance}.  It is thus expected that our empirical analysis of Gaussian data will be broadly illustrative.

As $\|\x_{1}-\x_{2}\|_{2}$ increases, Theorem \ref{thm:MainResult} becomes more informative.  In Figure \ref{fig:GaussianExamples}, samples are drawn from $\mu_{G}$ with different amounts of separation (and hence different $\delta, \lambda_{\numclust+1}, \kappa$ values) and we show the dependence on $\epsilon=\epsilon(t)$ of the bounds 

\begin{align*}
\overline{\Din}=& 2\frac{\epsilon(t)}{\sqrt{n}}\gamma(t),\\
\underline{\Dbtw}=&2\min_{y\in X}\|\sinf(y,\cdot)\|_{\ell^{2}(\nu)}-2\frac{\epsilon(t)}{\sqrt{n}}\gamma(t).
\end{align*}

and the permissible time interval $[\ln\left(\frac{2\kappa}{\epsilon(t)}\right)/\ln\left(\frac{1}{\lambda_{\numclust+1}}\right),\frac{\epsilon(t)}{2\delta}]$.  For Theorem \ref{thm:MainResult} to be meaningful, $\epsilon=\epsilon(t)$ must be such that simultaneously $\overline{\Din}<\underline{\Dbtw}$ and $\ln\left(\frac{2\kappa}{\epsilon(t)}\right)/\ln\left(\frac{1}{\lambda_{\numclust+1}}\right)<\frac{\epsilon(t)}{2\delta}$.  As $\epsilon\rightarrow0$, $\overline{\Din}<\underline{\Dbtw}$ holds if the clusters are internally well-connected and separated, as articulated in (\ref{eqn:epsilon0}), while $\ln\left(\frac{2\kappa}{\epsilon(t)}\right)/\ln\left(\frac{1}{\lambda_{\numclust+1}}\right)\rightarrow\infty$ and $\frac{\epsilon(t)}{2\delta}\rightarrow 0$; a similar but reversed dichotomy occurs as $\epsilon\rightarrow\infty$.  Figure \ref{fig:GaussianExamples} illustrates this tension between the hypotheses of Theorem \ref{thm:MainResult} and the strength of its conclusion.

\begin{figure}
\centering
\begin{subfigure}{.32\textwidth}
\includegraphics[width=\textwidth]{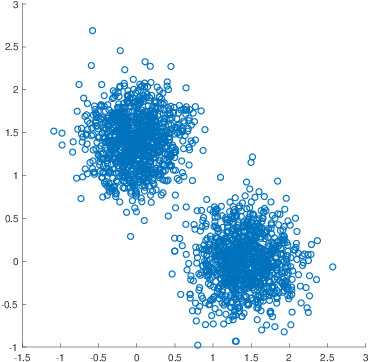}
\subcaption{$\x_{1}=(\sqrt{2},0)$, $\x_{2}=(0,\sqrt{2})$ }
\end{subfigure}
\begin{subfigure}{.32\textwidth}
\includegraphics[width=\textwidth]{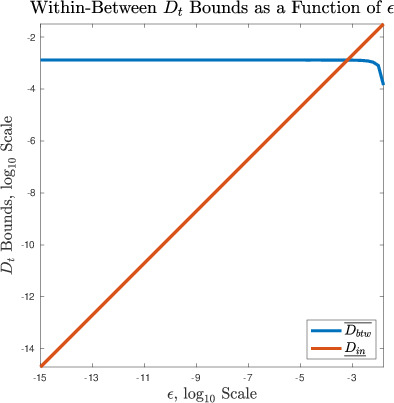}
\subcaption{$[\overline{\Din},\underline{\Dbtw}]$.}
\end{subfigure}
\begin{subfigure}{.32\textwidth}
\includegraphics[width=\textwidth]{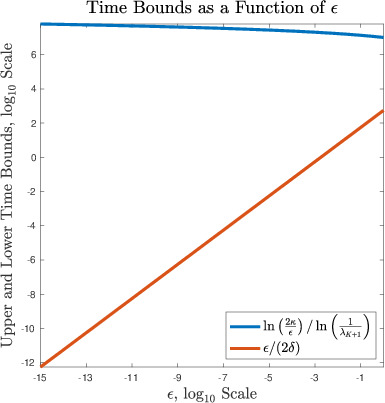}
\subcaption{$[\ln\left(\frac{2\kappa}{\epsilon(t)}\right)/\ln\left(\frac{1}{\lambda_{\numclust+1}}\right),\frac{\epsilon(t)}{2\delta}]$}
\end{subfigure}
\begin{subfigure}{.32\textwidth}
\includegraphics[width=\textwidth]{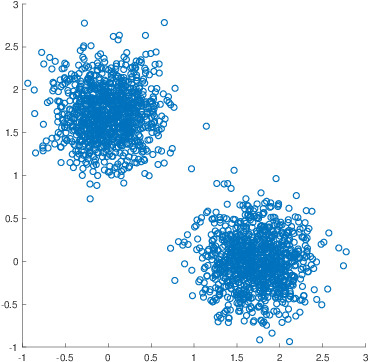}
\subcaption{$\x_{1}=(\sqrt{3},0)$, $\x_{2}=(0,\sqrt{3})$}
\end{subfigure}
\begin{subfigure}{.32\textwidth}
\includegraphics[width=\textwidth]{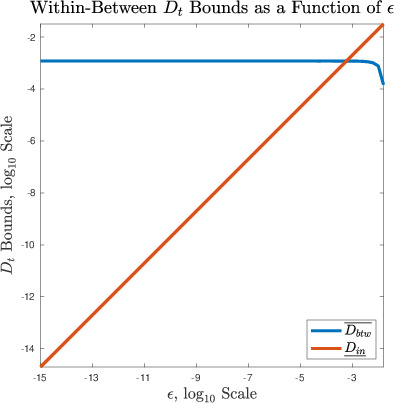}
\subcaption{$[\overline{\Din},\underline{\Dbtw}]$.}
\end{subfigure}
\begin{subfigure}{.32\textwidth}
\includegraphics[width=\textwidth]{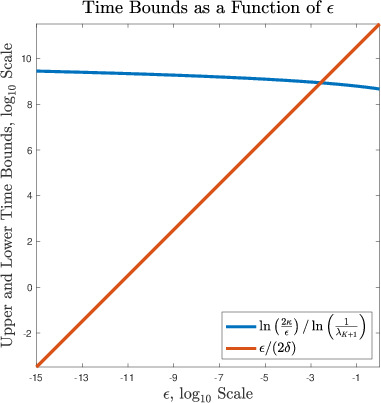}
\subcaption{$[\ln\left(\frac{2\kappa}{\epsilon(t)}\right)/\ln\left(\frac{1}{\lambda_{\numclust+1}}\right),\frac{\epsilon(t)}{2\delta}]$}
\end{subfigure}
\begin{subfigure}{.32\textwidth}
\includegraphics[width=\textwidth]{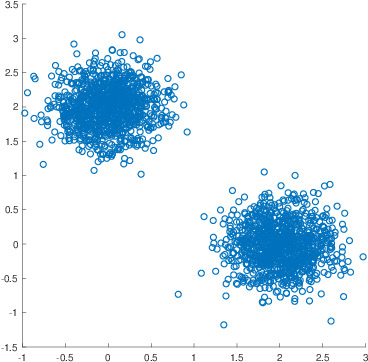}
\subcaption{$\x_{1}=(2,0)$, $\x_{2}=(0,2)$}
\end{subfigure}
\begin{subfigure}{.32\textwidth}
\includegraphics[width=\textwidth]{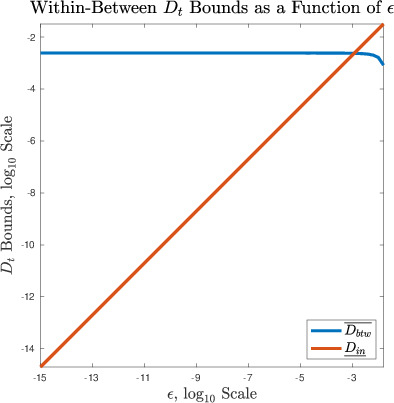}
\subcaption{$[\overline{\Din},\underline{\Dbtw}]$.}
\end{subfigure}
\begin{subfigure}{.32\textwidth}
\includegraphics[width=\textwidth]{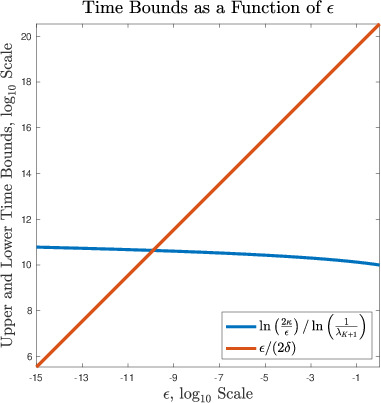}
\subcaption{$[\ln\left(\frac{2\kappa}{\epsilon(t)}\right)/\ln\left(\frac{1}{\lambda_{\numclust+1}}\right),\frac{\epsilon(t)}{2\delta}]$}
\end{subfigure}
\caption{\label{fig:GaussianExamples}  2000 data points sampled from $\mu_{G}$, for various means are shown in (a), (d), (g).  In (b), (e), (h), the range between $\overline{\Din}$ and $\underline{\Dbtw}$---as a function of $\epsilon$---is shown. Plots (c), (f), (i) show the $t$ interval guaranteed by Theorem \ref{thm:MainResult}, indicating the existence of a non-empty range of $t$ for which the conclusions of Theorem \ref{thm:MainResult} apply whenever the red curve is above the blue curve.  As the means move apart, the time interval in which Theorem \ref{thm:MainResult} guarantees good separation between the clusters expands.  This is makes sense intuitively, since as the clusters move apart, the unsupervised learning problem becomes easier.}
\end{figure}


\subsection{Relationship Between Time and Scaling Parameter in Diffusion Distances}
\label{s:TimeVsDD}
The Markov chain underlying diffusion distances is typically constructed using the heat kernel $\K(x,y)=e^{-\|x-y\|_{2}^{2}/\sigma^{2}}$, for some choice of (spatial) scale parameter $\sigma$.  Once $\Pm=\Pm_\sigma$ is constructed, the time parameter $t$ enters.  There exists an asymptotic relationship between $t$ and $\sigma$ as $n\rightarrow\infty$:

\begin{align}\label{eqn:TsigmaEquivalence}\lim_{\sigma\rightarrow 0} \Pm_{\sigma}^{T_{0}/\sigma}=e^{-T_{0}\Delta},\end{align}

where $e^{-T_{0}\Delta}$ is the infinitesimal generator corresponding to continuous diffusion with canonical time $T_{0}$ \citep{Lafon2006}.  So, asymptotically as $n\rightarrow\infty$, and requiring $\sigma\rightarrow 0$ and $t=\frac{T_{0}}{\sigma}$, using $(t,\sigma)$ is equivalent to using $(Ct,\sigma/C)$ for any constant $C>0$.  This suggests that asymptotically, the performance of LUND with respect to $\sigma, t$ should be constant if $t\sigma$ constant.  As we shall see observe in Section \ref{sec:NumericalExperiments}, working with {\em{finite}} data in the {\em{cluster}} setting, rather than the asymptotic regime on a common manifold, may lead to more subtle relationships between $t$ and $\sigma$.


\section{Performance Guarantees for Unsupervised Learning}\label{sec:PerformanceGuarantees}

We consider now how the LUND algorithm (Algorithm \ref{alg:LUND}) performs on data $X=\bigcup_{k=1}^{\numclust}X_{k}$.  Let $p(x)$ be a kernel density estimator for $x\in X$, let $\rho_{t}$ be as in (\ref{eqn:rho_t}), and recall $\Dt(x)=p(x)\rho_{t}(x)$.  The LUND algorithm sets the maximizers of $\Dt$ to be the modes of the clusters.  Requiring potential modes to have large $\rho_{t}$ values enforces that modes should be far in diffusion distance from other high-density points, and incorporating $p(x)$ downweighs outliers, which may be far in diffusion distance from their nearest neighbor of higher empirical density.

\begin{thm}\label{thm:Maximizers}  Suppose $X=\bigcup_{k=1}^{\numclust}X_{k}$.  If $\Din/\Dbtw<\min(\M)/\max(\M)$, then the $\numclust$ maximizers $\{\x_{i}^{*}\}_{i=1}^{\numclust}$ of $\Dt(x)$ are such that $\x_{i}^{*}$ is a highest empirical density point of $X_{k_{i}}$ for some permutation $(k_{1},\dots,k_{\numclust})$ of $(1,\dots,\numclust)$.
\end{thm}

\begin{proof}We proceed by induction on $1\le m\le \numclust$.  Clearly $\x_{1}^{*}=\argmax_{y\in X}p(y)$ is a highest empirical density point of some $X_{k}$.  Then suppose $\x_{1}^{*},\dots,\x_{m}^{*}, m<\numclust$, have been determined, and are highest empirical density points of distinct classes $X_{k_{1}},\dots,X_{k_{m}}$.  We show that $\x_{m+1}^{*}$ must be a highest density point among the remaining $X_{k}$, $k\notin\{ k_{1},\dots,k_{m}\}$.  

First, suppose $x\in X_{k_{r}}$ for some $r\in\{1,\dots, m\}$ is any point in the classes already discovered, not of maximal within-class density.  Then $\rho_{t}(x)\le\Din$, since $\x_{r}^{*}\in X_{k_{r}}$ has $p(x)<p(\x_{r}^{*})$, and hence for any $\x$ a highest density point in a cluster not already discovered,  $$\Dt(x)< p(\x_{k_{r}}^{*})\rho_{t}(x)\le\max(\M)\Din \le\frac{\max(\M)}{\min(\M)}\frac{\Din}{\Dbtw}p(\x)\rho_{t}(\x)<p(\x)\rho_{t}(\x)=\Dt(\x).$$  Hence, $\x_{m+1}^{*}\neq x$. 

Now, suppose $x\in X_{k}, k\neq k_{1},\dots k_{m}$.  If  $x\neq \x$ an empirical density maximizer of $X_{k}$, then: $$\Dt(x)=p(x)\rho_{t}(x)<p(\x)\rho_{t}(x)\le p(\x)\Din<p(\x)\Dbtw\le p(\x)\rho_{t}(\x)=\Dt(\x).$$  Hence, $\x_{m+1}^{*}\neq x,$ and thus $\x_{m+1}^{*}$ must be among the classwise empirical density maximizers of $X_{k}$, $k\notin\{k_{1},\dots,k_{m}\}$.
\end{proof}

We remark that requiring $\Din/\Dbtw<\min(\M)/\max(\M)$ allows for a simple detection of the modes, while a weaker hypothesis would allow for mode detection based on a more refined analysis of the decay of the sorted $\Dt(x)$ values.  

A similar method proves that the decay of $\Dt$ determines the number of clusters $\numclust$.  The problem of estimating the number of clusters is a crucial one, but few methods admit theoretical guarantees; see \citep{Little2015} for an overview.

\begin{cor}\label{cor:NumClusters}Let $\{x_{m_{i}}\}_{i=1}^{n}$ be the points $\{x_{i}\}_{i=1}^{n}$, sorted in non-increasing order: $\Dt(x_{m_{1}})\ge \Dt(x_{m_{2}})\ge \dots \ge \Dt(x_{m_{n}})$.  Then:

\begin{enumerate}[(a)]
\item $\frac{\Dt(x_{m_{j}})}{\Dt(x_{m_{j+1}})}\le \frac{\max(\M)}{\min(\M)}\frac{\max_{i=1,\cdots,\numclust}\rho_{t}(x_{m_{i}})}{\min_{i=1,\cdots,\numclust}\rho_{t}(x_{m_{i}})}$ for $j< \numclust$.
\item $ \frac{\Dt(x_{m_{\numclust}})}{\Dt(x_{m_{\numclust+1}})}\ge  \frac{\min(\M)}{\max(\M)}\frac{\Dbtw}{\Din}$.
\end{enumerate}
\end{cor}

\begin{proof}Statement (a) is immediate from the definition.  To see (b), we compute $$\Dt(x_{m_{\numclust}})\ge \min(\M)\Dbtw= \frac{\min(\M)}{\max(\M)}\frac{\Dbtw}{\Din}\Din\max(\M)\ge  \frac{\min(\M)}{\max(\M)}\frac{\Dbtw}{\Din}\Dt(x_{m_{\numclust+1}}).$$

\end{proof}

Hence if $ \frac{\Din}{\Dbtw}\frac{\max_{i=1,\cdots,\numclust}\rho_{t}(x_{m_{i}})}{\min_{i=1,\cdots,\numclust}\rho_{t}(x_{m_{i}})}\ll \left(\frac{\min(\M)}{\max(\M)}\right)^{2},$ there will be a sharp drop-off in the values of $\Dt$ after the first $\numclust$ maximizers.   This observation can be used to accurately identify the number of clusters.  The ratio $\min_{i=1,\cdots\numclust}\rho_{t}(x_{m_{i}})/\max_{i=1,\cdots\numclust}\rho_{t}(x_{m_{i}})$ will be insignificant unless the clusters are arranged at different scales (i.e. some clusters are very close to each other but far from others).  Similarly, $({\min(\M)}/{\max(\M)})^{2}$ will be nearly 1 if the maximal densities of the clusters are comparable.

Once the modes have been learned correctly, points may be clustered simply by labeling each mode as belonging to its own class, then requiring that every point has the same label as its nearest neighbor in diffusion distance of higher density.

\begin{cor}\label{cor:PerformanceGuarantees}
Suppose $X=\bigcup_{k=1}^{\numclust}X_{k}$.  Let $\{x_{m_{i}}\}_{i=1}^{n}$ be the points $\{x_{i}\}_{i=1}^{n}$, sorted so that $\Dt(x_{m_{1}})\ge \Dt(x_{m_{2}})\ge \dots \ge \Dt(x_{m_{n}})$.  Then Algorithm \ref{alg} labels all points correctly for any $\tau$ satisfying $$\frac{\Din}{\Dbtw}\frac{\max_{i=1,\cdots,\numclust}\rho_{t}(x_{m_{i}})}{\min_{i=1,\cdots,\numclust}\rho_{t}(x_{m_{i}})}<\tau<\left(\frac{\min(\M)}{\max(\M)}\right)^{2}.$$
\end{cor}
\begin{proof}
By Corollary \ref{cor:NumClusters}, the algorithm correctly estimates $\hat{K}$.  Then, by Theorem \ref{thm:Maximizers}, the algorithm correctly learns the empirical density maximizers of each of the $\{X_{k}\}_{k=1}^{\numclust}$.  It remains to show that the subsequent labeling of all points is accurate.  For an unlabeled point $x\in X_{k}$, its nearest diffusion neighbor of higher density, $x^{*}$, must be in the same cluster $X_{k}$, since $\Din<\Dbtw$.  Moreover, that point is already labeled as $Y(x^{*})=k$, since $p(x^{*})\ge p(x)$.  Hence, $Y(x)=k$ and by induction, all points are labeled correctly.  
\end{proof}

The dependence on $\tau$ is somewhat unsatisfying, and in practice, this quantity can be removed from the inputs of Algorithm \ref{alg} by instead setting $\hat{\numclust}=\argmax_{k}\Dt(x_{m_{k}})/\Dt(x_{m_{k}+1})$.  This provably detects $\numclust$ accurately by noting that the ratios $\Dt(x_{m_{j+1}})/\Dt(x_{m_{j}})$ are small for $j>\numclust$ under a range of reasonable assumptions, for example the assumptions that the density of each cluster is bounded away from 0 and the ratio of the minimal and maximal within-cluster diffusion distance is bounded.

If $\hat{K}$ is known a priori, a weaker condition guarantees correct labeling:

\begin{cor}Suppose $X=\bigcup_{k=1}^{\numclust}X_{k}$ and $\numclust$ is known.  If $\frac{\Din}{\Dbtw}<\frac{\min(\M)}{\max(\M)}$, then Algorithm \ref{alg}, using $\numclust$ instead of $\hat{\numclust}$ for the number of clusters, labels all points correctly.
\end{cor}

\begin{proof}This follows from Theorem \ref{thm:Maximizers}, along with $\Din<\Dbtw$.
\end{proof}


\section{Numerical Experiments}\label{sec:NumericalExperiments}

We return to the motivating datasets of Section \ref{sec:DataModel}.  The diffusion distances are computed by truncating (\ref{eqn:EigenfunctionsDD}) by summing only over the largest $M=100\ll n$ eigenpairs, and the kernel density estimator $p(x)$ uses 100 nearest neighbors.  

We compute a number of statistics on the data to test our theoretical estimates and to verify the efficacy of the proposed algorithm.  For the first two datasets we examine, we plot $\Din, \Dbtw$ as functions of $t$, to observe the multitemporal nature of our clustering algorithm.  We also plot the diffusion distances from a fixed point for a variety of $t$ values, to illustrate the multitemporal behavior of these distances. We also compute the theoretical estimates on $\|\Pm^{t}-\Sm^{\infty}\|_{\infty}$ as guaranteed by Theorem \ref{thm:NearReducibility}.  The tightness of the theoretical estimates is evaluated by comparing to the empirical values.  

After these evaluations, we cluster the data with the proposed LUND algorithm and compute the accuracy, comparing with spectral clustering and the FSFDPC algorithm.  We moreover compute the estimates of $\numclust$ with both the proposed method $\hat{\numclust}=\argmax_{k}\Dt(x_{m_{k}})/\Dt(x_{m_{k}+1})$ where $\{x_{m_{i}}\}_{i=1}^{n}$ are the points $\{x_{i}\}_{i=1}^{n}$ sorted so that $\Dt(x_{m_{1}})\ge \Dt(x_{m_{2}})\ge \dots \ge \Dt(x_{m_{n}})$, and spectral clustering eigengap, as a function of the crucial parameters of the respective algorithms.  For spectral clustering, we consider the variant in which just the second eigenvector $\psi_{2}$ is used \citep{Shi2000}, as well as the variant in which the first $\numclust$ eigenvectors \{$\psi_{i}\}_{i=1}^{\numclust}$ are used \citep{Ng2002}.  All experiments are conducted on randomly generated data, with results averaged over 100 trials.


\subsection{Bottleneck Data}

We first analyze the linear, multimodal dataset of Figure \ref{fig:MotivatingDatasets}, in which two of the clusters feature two high-density regions, connected by a lower density bottleneck region.  Theorem \ref{thm:NearReducibility} upper bounds $\|\Pm^{t}-\Sm^{\infty}\|_{\infty}<\epsilon(t)$ in terms of $\delta, \lambda_{\numclust+1}, \kappa$, which for this data have values $\delta=6.2697\times10^{-8}, \ \lambda_{\numclust+1}=1-1.7563\times 10^{-4}, \ \kappa=  2.6738\times10^{2}$ when $\Pm$ is constructed with $\sigma=.15$.  As shown in Figure \ref{fig:BottleneckNoiselessTimePlot}, the theoretical estimate correctly illustrates the overall behavior of the transition from initial distribution, to mesoscopic equilbria, then to a global equilibrium.  

The distance from a high-density point across time scales appears in Figure \ref{fig:BottleneckNoiseless}.  For small time values, the diffusion distance scales similarly to Euclidean distance.  However, by time $t=10^{8}$, a mesoscopic equilibrium has been reached, and all points in the cluster are rather close together.  By $t=10^{16}$, a global equilibrium has been reached.

\begin{figure}[!htb] 
\centering
\begin{subfigure}{.49\textwidth}
\includegraphics[width=\textwidth]{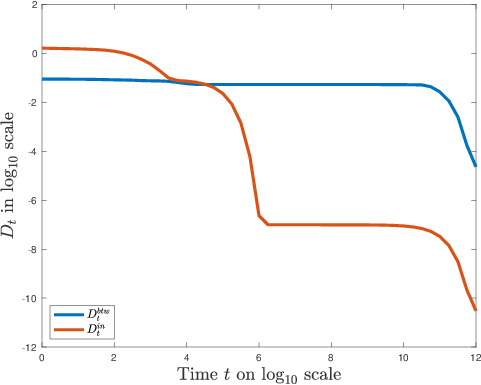}
\subcaption{Plot of $\Dbtw, \Din$ against $t$.}
\end{subfigure}
\begin{subfigure}{.49\textwidth}
\includegraphics[width=\textwidth]{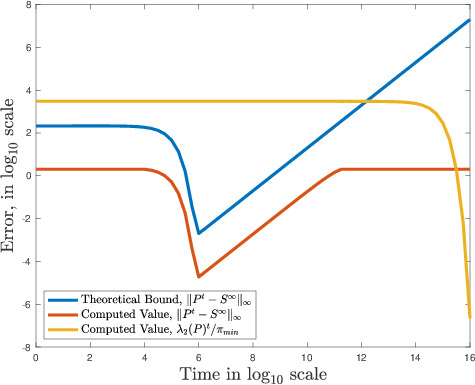}
\subcaption{Quantities related to mesoscopic and global equilibria.}
\end{subfigure}
\caption{\label{fig:BottleneckNoiselessTimePlot}  In (a), $\Dbtw, \Din$ are plotted against $t$.  For $t<10^{4}$, $\Din>\Dbtw$, since for small time, $D_{t}$ is essentially the same as Euclidean distance.  Around $t=10^{4}$, there is a transition, in which $\Din\ll \Dbtw$.  This corresponds to the Markov chain reaching mesoscopic equilibria in which the chain is well-mixed on each cluster, but not well-mixed globally.  In (b), we plot three quantities against $t$: the theoretical bound on $\|\Pm^{t}-\Sm^{\infty}\|_{\infty}$ guaranteed by Theorem \ref{thm:NearReducibility}; the actual, empirically computed quantity $\|\Pm^{t}-\Sm^{\infty}\|_{\infty}$; and the empirically computed quantity $\lambda_{2}(\Pm)^{t}/\boldpi_{\text{min}}$, which estimates the distance to the stationary distribution.  Notice that $\|\Pm^{t}-\Sm^{\infty}\|_{\infty}$ gets small, both the theoretical bound and the empirical value, around $t=10^{5}$.  It then increases.  Around $t=10^{14}$, $\lambda_{2}(\Pm)^{t}/\boldpi_{\text{min}}$ decays exponentially to 0, indicating that the global equilibrium has been reached.  Note that the theoretical estimate on $\|\Pm^{t}-\Sm^{\infty}\|_{\infty}$ is not tight, though it accurately captures the overall behavior of the quantity with respect to $t$.  }
\end{figure}

\begin{figure}[b] 
\centering
\begin{subfigure}{.24\textwidth}
\includegraphics[width=\textwidth]{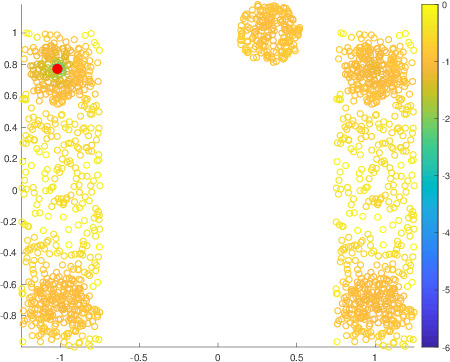}
\subcaption{$t=0$}
\end{subfigure}
\begin{subfigure}{.24\textwidth}
\includegraphics[width=\textwidth]{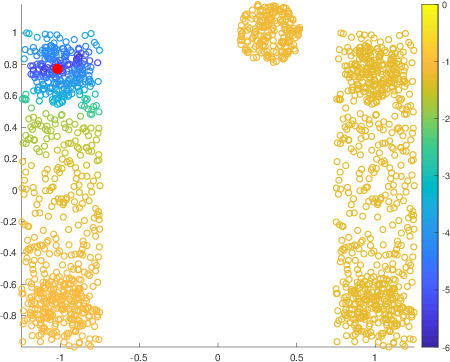}
\subcaption{$t=10^4$}
\end{subfigure}
\begin{subfigure}{.24\textwidth}
\includegraphics[width=\textwidth]{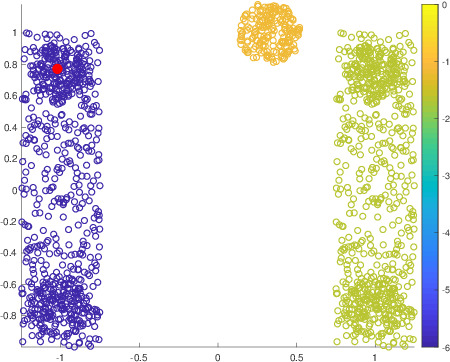}
\subcaption{$t=10^8$}
\end{subfigure}
\begin{subfigure}{.24\textwidth}
\includegraphics[width=\textwidth]{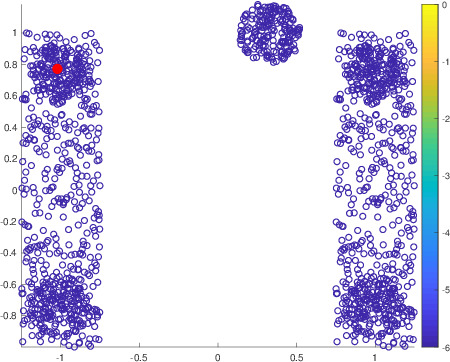}
\subcaption{$t=10^{16}$}
\end{subfigure}
\caption{\label{fig:BottleneckNoiseless}  A high-density point is shown in red, and all other points are colored by $D_{t}$ distance from this point in $\log_{10}$ scale.  The transition from initial distribution (a), to mesoscopic equilibrium (c), to global equilibrium (d), is illustrated as $t$ grows.}
\end{figure}

\subsubsection{Bottleneck Data Clustering Evaluation}

Comparisons with spectral clustering appear in Figures \ref{fig:BottleneckAccuracy} and \ref{fig:BottleneckEstimateK}.  In Figure \ref{fig:BottleneckEstimateK}, it is shown that for all values of the spatial scale parameter $\sigma$, the eigengap estimated number of clusters $\hat K$ is $1$, i.e. always incorrect. On the other hand, the figure shows that there is a range of $(\sigma,t)$ values---mesoscopic in $t$---for which LUND achieves perfect accuracy.  Indeed, after an initial phase in which the number of clusters is estimated as 1, the LUND estimate for $\numclust$ is decreasing in $t$, corresponding to the mixing of different clusters over time.  

\begin{figure}[!htb]
\centering
\begin{subfigure}{.47\textwidth}
\includegraphics[width=\textwidth]{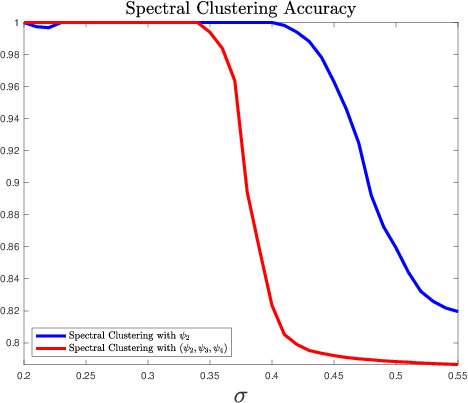}
\subcaption{Accuracy spectral clustering}
\end{subfigure}
\begin{subfigure}{.51\textwidth}
\includegraphics[width=\textwidth]{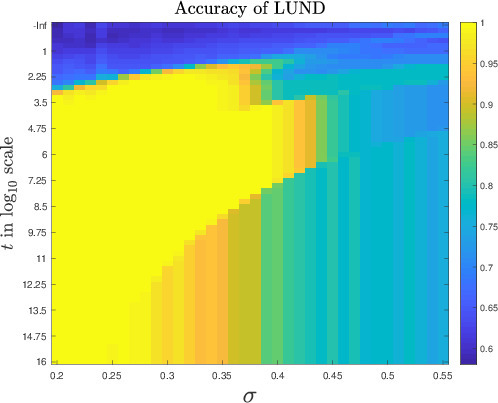}
\subcaption{Accuracy of LUND}
\end{subfigure}
\caption{\label{fig:BottleneckAccuracy}  Accuracy of two variations of spectral clustering compared to LUND as functions of $\sigma$.  While spectral clustering with $\psi_{2}$ performs nearly perfectly for $\sigma<.4$, its performance degrades as $\sigma$ increases.  Classical spectral clustering using $\psi_{1},\psi_{2},\psi_{3}$ achieves perfect clustering of the data for roughly $\sigma<.35$.  LUND is able to achieve perfect clustering accuracy for a wide range of $(\sigma,t)$ pairs, mainly for those $\sigma$ values which allows spectral clustering with just $\psi_{2}$ to succeed.  As $\sigma$ increases, the mesoscopic regime in which perfect accuracy is achieved shrinks before disappearing entirely around $\sigma=.45$. In this data, spectral clustering with just $\psi_{2}$ performs about as well as LUND in terms of accuracy, assuming $\numclust$ is known.}
\end{figure}

\begin{figure}[!htb] 
\centering
\begin{subfigure}{.24\textwidth}
\includegraphics[width=\textwidth]{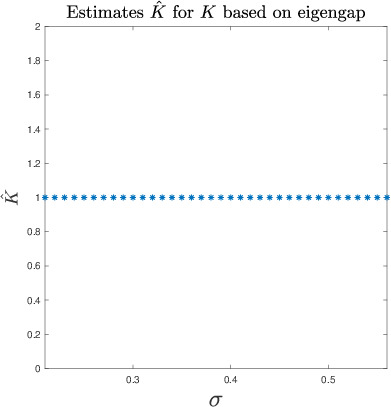}
\subcaption{Estimates of $\hat{K}$ using eigengap statistic.}
\end{subfigure}
\begin{subfigure}{.24\textwidth}
\includegraphics[width=\textwidth]{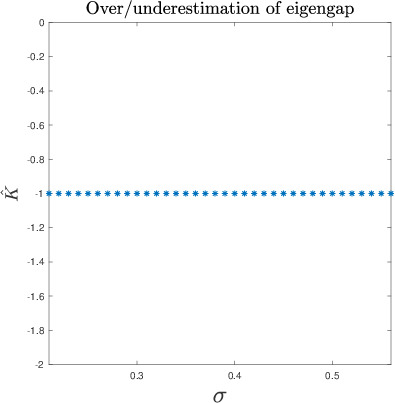}
\subcaption{Over/underestimation of $\hat{K}$ using eigengap statistic.}
\end{subfigure}
\begin{subfigure}{.24\textwidth}
\includegraphics[width=\textwidth]{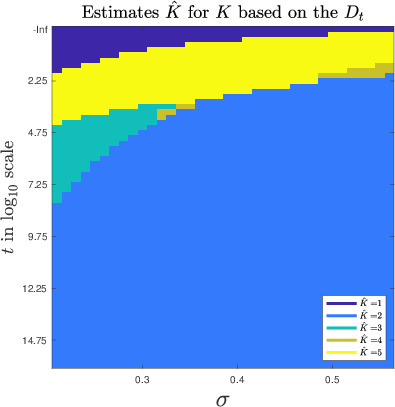}
\subcaption{Estimates of $\hat{K}$ using $\Dt$ statistic.}
\end{subfigure}
\begin{subfigure}{.24\textwidth}
\includegraphics[width=\textwidth]{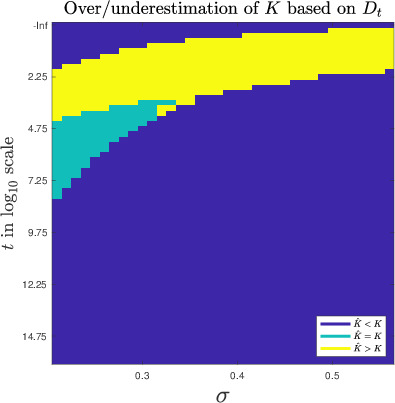}
\subcaption{Over/underestimation of $\hat{K}$ using $\Dt$ statistic.}
\end{subfigure}
\caption{\label{fig:BottleneckEstimateK}In (a), we see the estimates of $\hat{K}$ using the eigengap statistic, as a function of spatial scale parameter $\sigma$.  The eigengap consistently estimates $\hat{K}=1<3=K$, indicating that the multimodal nature of this data is too complicated for the spectral clustering eigengap to handle.  A quantized version of these estimates is shown in (b), in which entry 0 indicates correct estimation, -1 indicates $\hat{K}<K$ and 1 indicates $\hat{K}>K$.  
	There is a regime of $(\sigma,t)$ values in which LUND correctly estimates the number of clusters, as shown in (c) and (d).  This regime is essentially, after an initial time, monotonic decreasing in $t$, and the mesoscopic region in which $\hat{K}=K$ is decreasing in $\sigma$.}
\end{figure}

The LUND algorithm and FSFDPC are compared in Figure \ref{fig:BottleneckLUNDversusFSFDPC}.  Due to the non-spherical shapes of the clusters, FSFDPC is unable to learn the modes of the data correctly, and consequently assigns modes to the same cluster: the modes learned by FSFDPC and subsequent labels appear in subfigures (a), (b), respectively.  In contrast, LUND learns one mode from each cluster, as shown in (c).  Consequently, all points are labeled correctly, as shown in (d).

\begin{figure}[!htb] 
\centering
\begin{subfigure}{.24\textwidth}
\includegraphics[width=\textwidth]{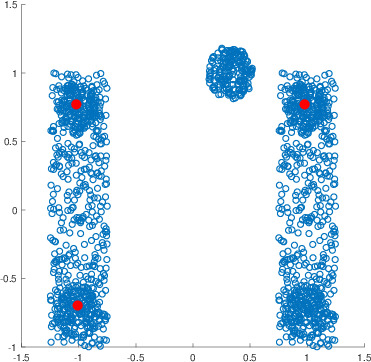}
\subcaption{Learned modes, FSFDPC}
\end{subfigure}
\begin{subfigure}{.24\textwidth}
\includegraphics[width=\textwidth]{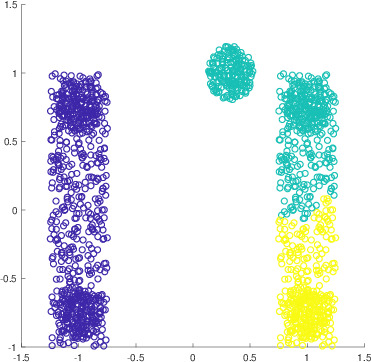}
\subcaption{Learned labels, FSFDPC}
\end{subfigure}
\begin{subfigure}{.24\textwidth}
\includegraphics[width=\textwidth]{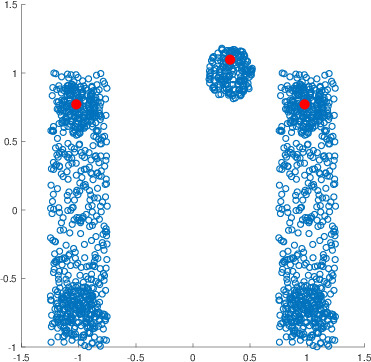}
\subcaption{Learned modes, LUND}
\end{subfigure}
\begin{subfigure}{.24\textwidth}
\includegraphics[width=\textwidth]{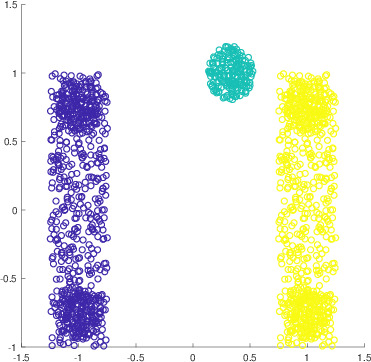}
\subcaption{Learned labels, LUND}
\end{subfigure}
\caption{\label{fig:BottleneckLUNDversusFSFDPC}  Comparison of FSFDPC to LUND.  In (a), the modes learned from FSFDPC---with Euclidean distances---are plotted.  Due to the eloganted, non-spherical nature of the data, the modes are learned incorrectly.  The subsequent labels, shown in (b), illustrate FSFDPC is not able to capture the structure of this data.  In (c), the modes learned from LUND are shown.  One mode is learned from each cluster, which allows for a correct labeling of all data points with LUND, as shown in (d).  LUND used parameters $(\sigma,t)=(.15,10^{6})$ for these data.}
\end{figure}


\subsection{Nonlinear Data}

We now consider the nonlinear multimodal data in the form of circles from Figure \ref{fig:MotivatingDatasets}.  The innermost circle is filled-in, and has only one high-density region.  It is surrounded by two circles, each with two high-density regions connected by low-density regions.  The paths connecting antipodal points on the outer circles are long, which suggests these sets will have low conductance.  In the context of Theorem \ref{thm:NearReducibility}, the parameters for this data have values $\delta=1.7225\times10^{-4}, 1-\lambda_{\numclust+1}= 6.8350\times 10^{-5}, \kappa=2.655\times10^{2}$ with $\sigma=.175$.  Comparison of theoretical and empirical estimates appear in Figure \ref{fig:NonlinearNoiselessTimePlot}, and the diffusion distances from one of the high-density points appear in Figure \ref{fig:NonlinearNoiseless}, illustrating the transition from initial distribution, to mesoscopic equilibrium, to global equilibrium.

\begin{figure}[!htb] 
\centering
\begin{subfigure}{.49\textwidth}
\includegraphics[width=\textwidth]{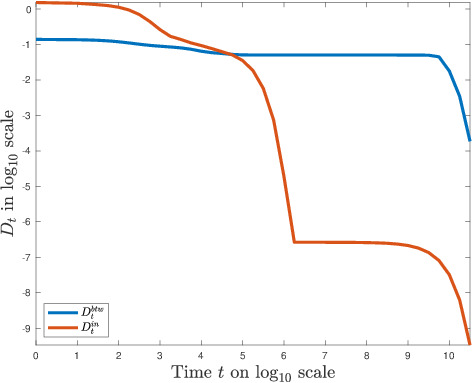}
\subcaption{Plot of $\Dbtw, \Din$ against $t$.}
\end{subfigure}
\begin{subfigure}{.49\textwidth}
\includegraphics[width=\textwidth]{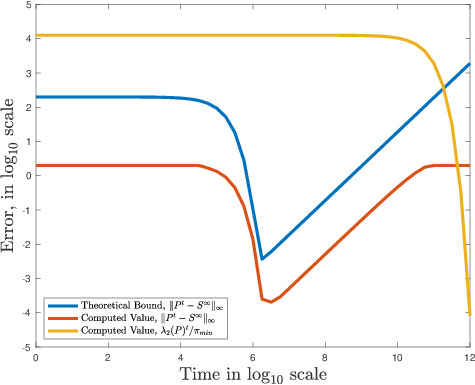}
\subcaption{Quantities related to mesoscopic and global equilibria.}
\end{subfigure}
\caption{\label{fig:NonlinearNoiselessTimePlot} In (a), we plot $\Dbtw, \Din$ against $t$.  For roughly $t<10^{5}$, $\Din>\Dbtw$; around $t=10^{5}$, there is a transition, in which $\Din\ll \Dbtw$.  This corresponds to the Markov chain reaching mesoscopic equilibria in which the chain is well-mixed on each cluster, but not well-mixed globally.  In (b), we plot three quantities against $t$: the theoretical bound on $\|\Pm^{t}-\Sm^{\infty}\|_{\infty}$ guaranteed by Theorem \ref{thm:NearReducibility}; the actual, empirically computed quantity $\|\Pm^{t}-\Sm^{\infty}\|_{\infty}$; and the empirically computed quantity $\lambda_{2}(\Pm)^{t}/\boldpi_{\text{min}}$, which estimates the distance to the stationary distribution.  Notice that $\|\Pm^{t}-\Sm^{\infty}\|_{\infty}$ gets small, both the theoretical bound and the empirical value, around $t=10^{4.5}$.  It then increases.  Around $t=10^{10}$,  $\lambda_{2}(\Pm)^{t}/\boldpi_{\text{min}}$ decays exponentially to 0, indicating that the global equilibrium has been reached.}
\end{figure}

\begin{figure}[!htb] 
\centering
\begin{subfigure}{.24\textwidth}
\includegraphics[width=\textwidth]{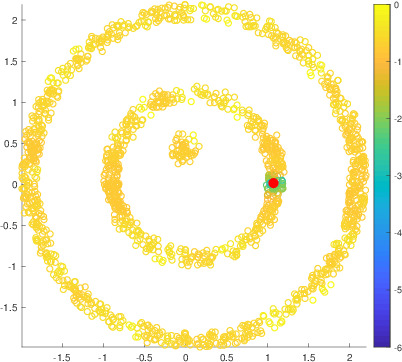}
\subcaption{$t=0$}
\end{subfigure}
\begin{subfigure}{.24\textwidth}
\includegraphics[width=\textwidth]{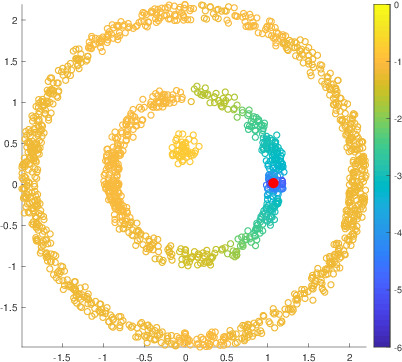}
\subcaption{$t=10^4$}
\end{subfigure}
\begin{subfigure}{.24\textwidth}
\includegraphics[width=\textwidth]{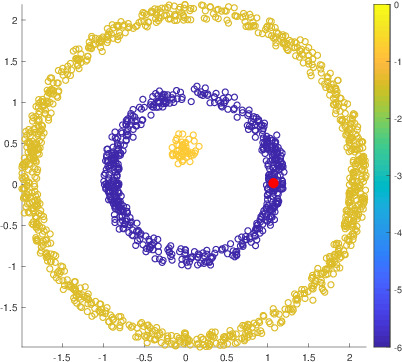}
\subcaption{$t=10^8$}
\end{subfigure}
\begin{subfigure}{.24\textwidth}
\includegraphics[width=\textwidth]{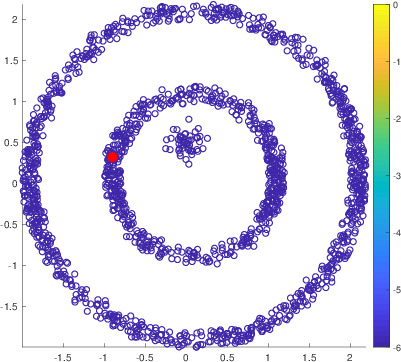}
\subcaption{$t=10^{16}$}
\end{subfigure}
\caption{\label{fig:NonlinearNoiseless}A high-density point is shown in red, and all other points are colored by $D_{t}$ distance from this point in $\log_{10}$ scale.  The transition from initial distribution (a), to mesoscopic equilibrium (c), to global equilibrium (d), is illustrated as $t$ grows.}
\end{figure}

\subsubsection{Nonlinear Data Clustering Evaluation}

LUND is compared with the two spectral clustering variants in Figures \ref{fig:NonlinearAccuracy} and \ref{fig:NonlinearEstimateK}.  In terms of overall accuracy, LUND with the correct choice of $t$ outperforms both methods of spectral clustering---using $\psi_{2}$ only and using $\psi_{1},\psi_{2},\psi_{3}$---for a range of $\sigma$ values.  The strong performance of LUND in the mesoscopic range, away from $t=0, t=\infty$, confirms the theoretical results, and demonstrates LUND's flexibility compared to classical spectral methods.  Beyond accuracy, the LUND estimator for $\numclust$ is empirically effective for a range of $(\sigma,t)$ values, while the eigengap is much less effective.  

\begin{figure}[!htb] 
\centering
\begin{subfigure}{.47\textwidth}
\includegraphics[width=\textwidth]{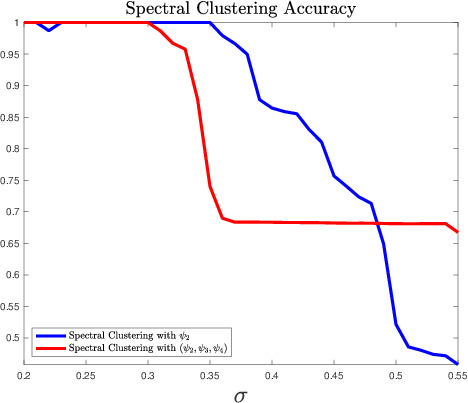}
\subcaption{Accuracy spectral clustering}
\end{subfigure}
\begin{subfigure}{.51\textwidth}
\includegraphics[width=\textwidth]{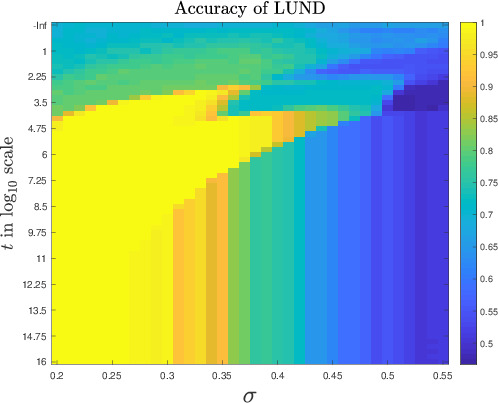}
\subcaption{Accuracy of LUND}
\end{subfigure}
\caption{\label{fig:NonlinearAccuracy}  Accuracy of two variations of spectral clustering compared to LUND as functions of $\sigma$.  While spectral clustering with $\psi_{2}$ performs well for very small $\sigma$, its performance degrades as $\sigma$ increases; classical spectral clustering using $\psi_{1},\psi_{2},\psi_{3}$ performs similarly though for a smaller range of $\sigma$.  LUND is able to achieve perfect clustering accuracy for a wide range of $(\sigma,t)$ pairs, in particular for pairs $(\sigma,t)$ such that spectral clustering fails.  As $\sigma$ increases, the mesoscopic regime in which perfect accuracy is achieved shrinks before disappearing entirely around $\sigma=.4$.  LUND outperforms spectral clustering with $(\psi_{2},\psi_{3},\psi_{4})$ roughly for $\sigma\in(.3,.4)$, and outperforms spectral clustering with $\psi_{2}$ alone roughly for $\sigma\in (.35,.4)$.}
\end{figure}

\begin{figure}[!htb] 
\centering
\begin{subfigure}{.24\textwidth}
\includegraphics[width=\textwidth]{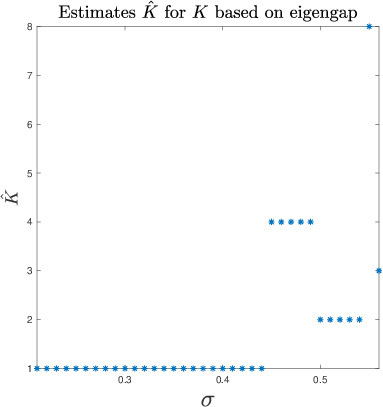}
\subcaption{Estimates of $\hat{K}$ using eigengap statistic.}
\end{subfigure}
\begin{subfigure}{.24\textwidth}
\includegraphics[width=\textwidth]{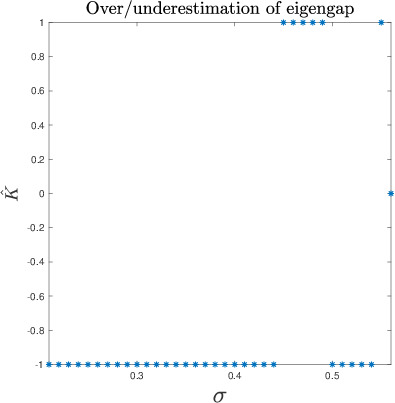}
\subcaption{Over/underestimation of $\hat{K}$ using eigengap statistic.}
\end{subfigure}
\begin{subfigure}{.24\textwidth}
\includegraphics[width=\textwidth]{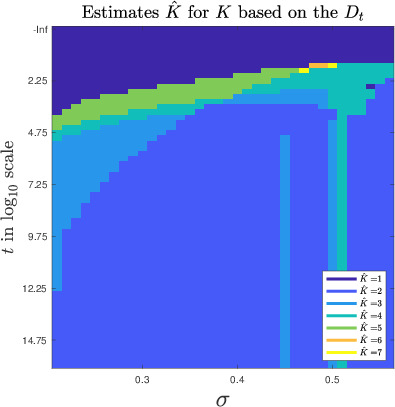}
\subcaption{Estimates of $\hat{K}$ using $\Dt$ statistic.}
\end{subfigure}
\begin{subfigure}{.24\textwidth}
\includegraphics[width=\textwidth]{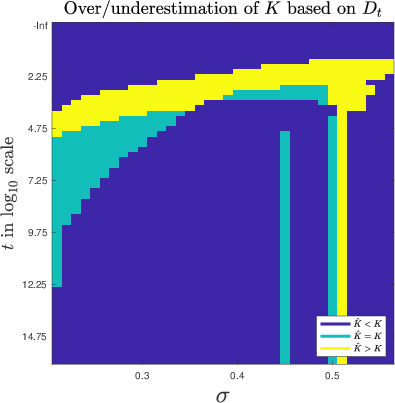}
\subcaption{Over/underestimation of $\hat{K}$ using $\Dt$ statistic.}
\end{subfigure}
\caption{\label{fig:NonlinearEstimateK}In (a) and (b), we see the eigengap consistently misestimates the number of clusters in the data, first estimating $\hat{K}=1$, before oscillating between various numbers of clusters, including ending on $\hat{K}=3$ for one value of $\sigma$.  LUND is able to achieve an accurate estimate for $\hat{K}$ for a range of $(\sigma,t)$ values, with generally more $t$ values yielding a correct estimate for smaller $\sigma$.}
\end{figure}

In Figure \ref{fig:NonlinearLUNDversusFSFDPC}, LUND is compared to FSFDPC.  LUND correctly learns the modes of the data and labels points correctly, as shown in (c), (d).  FSFDPC, however, fails to learn the modes of the data correctly, leading to erroneous labeling---see subfigures (a), (b) respectively.
\begin{figure}[!htb] 
\centering
\begin{subfigure}{.24\textwidth}
\includegraphics[width=\textwidth]{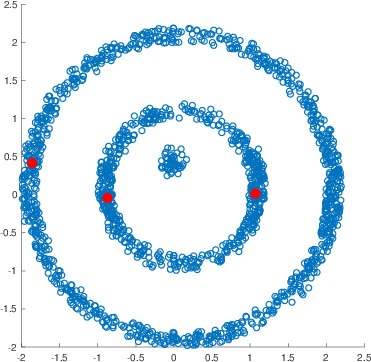}
\subcaption{Learned modes, FSFDPC}
\end{subfigure}
\begin{subfigure}{.24\textwidth}
\includegraphics[width=\textwidth]{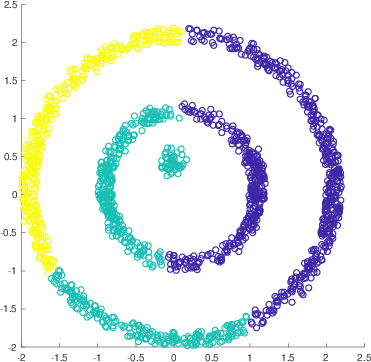}
\subcaption{Learned labels, FSFDPC}
\end{subfigure}
\begin{subfigure}{.24\textwidth}
\includegraphics[width=\textwidth]{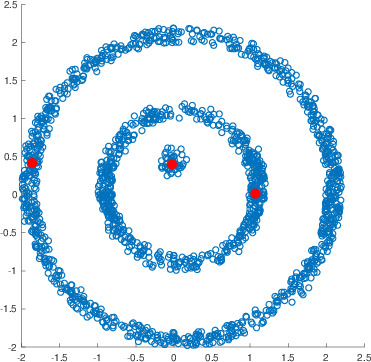}
\subcaption{Learned modes, LUND}
\end{subfigure}
\begin{subfigure}{.24\textwidth}
\includegraphics[width=\textwidth]{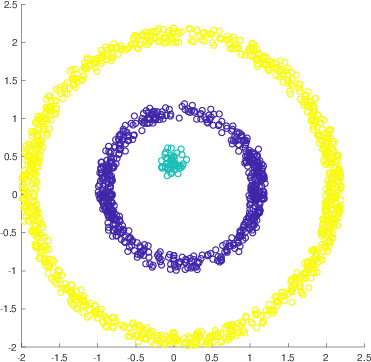}
\subcaption{Learned labels, LUND}
\end{subfigure}
\caption{\label{fig:NonlinearLUNDversusFSFDPC}In (a), the modes learned from FSFDPC---with Euclidean distances---are plotted.  The nonlinear nature of the data causes the modes to be learned incorrectly.  The subsequent labels, shown in (b), illustrate FSFDPC is not able to capture the structure of this data.  In (c), the modes learned from LUND are shown.  One mode is learned from each cluster, which allows for a correct labeling of all data points with LUND, as shown in (d).  LUND used parameters $(\sigma,t)=(.175,10^{5})$ for these data.}
\end{figure}


\subsection{Gaussian Data}

As a final synthetic example, we consider the Gaussians of Figure \ref{fig:NadlerGaussians}, which were constructed by \citet{Nadler2007_Fundamental} to be data on which both variants of spectral clustering fail.  These data are not sufficiently separated for Theorem \ref{thm:MainResult} to apply, but LUND still is able to perform well, owing to the incorporation of density, which allows to easily estimate the modes of the data.  Comparisons to spectral clustering in terms of overall accuracy are reported in Figure \ref{fig:GaussiansOA}.  It is also enlightening to consider performances of LUND and spectral clustering in terms of \emph{average accuracy}, in which the overall accuracy on each of the clusters is computed separately, and these class-wise accuracies are then averaged.  Compared to the overall accuracy measure, the average accuracy measure discounts large clusters and increases the significance of small clusters.  Results for average accuracy are in Figure \ref{fig:GaussiansAA}.

\begin{figure}[!htb] 
\centering
\begin{subfigure}{.47\textwidth}
\includegraphics[width=\textwidth]{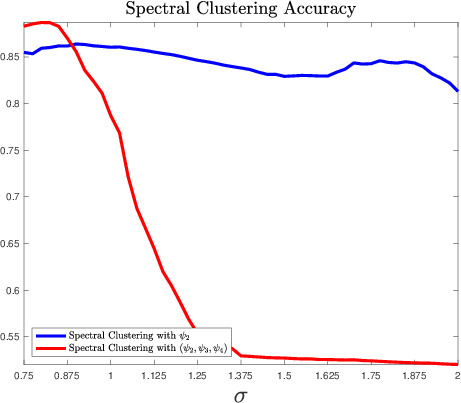}
\subcaption{Accuracy of spectral clustering}
\end{subfigure}
\begin{subfigure}{.51\textwidth}
\includegraphics[width=\textwidth]{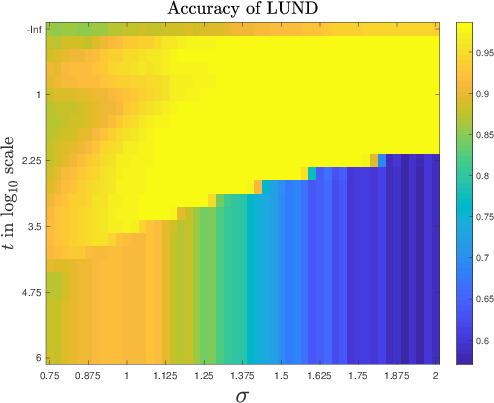}
\subcaption{Accuracy of LUND}
\end{subfigure}
\caption{\label{fig:GaussiansOA}The \emph{overall accuracy} of the two spectral clustering variants, as well as LUND, are shown for the Gaussian data.  In terms of overall accuracy, LUND is able to achieve near-perfect results for a range of parameter values.  Nearly all errors made were due to a point being generated from one Gaussian and landing very close to another Gaussian, which is essentially an unavoidable identifiability issue from which any unsupervised learning method would suffer.  Neither of the spectral clustering methods is able to match LUND's performance, which can be attributed to fundamental issues with the use of only the first small number of eigenvectors when performing spectral clustering, as shown by \citet{Nadler2007_Fundamental} and illustrated in Figure \ref{fig:NadlerGaussians}.}
\end{figure}

\begin{figure}[!htb] 
\centering
\begin{subfigure}{.47\textwidth}
\includegraphics[width=\textwidth]{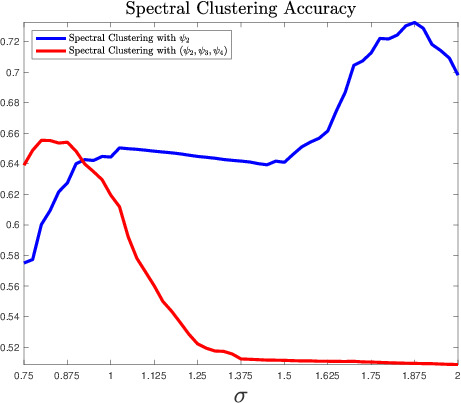}
\subcaption{Accuracy of spectral clustering}
\end{subfigure}
\begin{subfigure}{.51\textwidth}
\includegraphics[width=\textwidth]{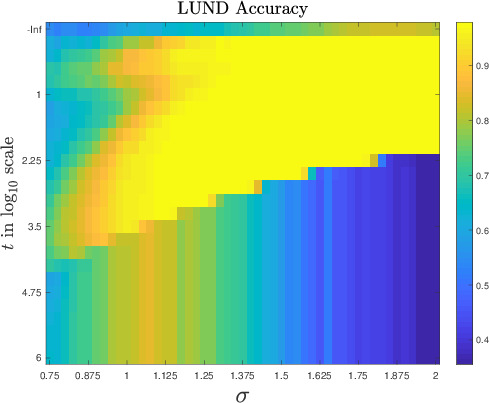}
\subcaption{Accuracy of LUND}
\end{subfigure}
\caption{\label{fig:GaussiansAA}The \emph{average accuracy} of the two spectral clustering variants, as well as LUND, are shown for the Gaussian data.  The results are qualitatively similar to overall accuracy, but with reduced performance for spectral clustering, since most of the errors made by the spectral clustering variants are on the small cluster, which is washed out by spectral clustering.  LUND achieves essentially perfect performance for a range of parameter values, excepting identifiability issues.}
\end{figure}

In Figure \ref{fig:GaussiansLUNDversusFSFDPC}, LUND is compared to FSFDPC.  LUND correctly learns the modes of the data and labels points correctly, as shown in (c), (d).  FSFDPC also learns the modes correctly, due to the unimodality of the Gaussian clusters and their isotropic covariance matrices.  

\begin{figure}[!htb] 
\centering
\begin{subfigure}{.24\textwidth}
\includegraphics[width=\textwidth]{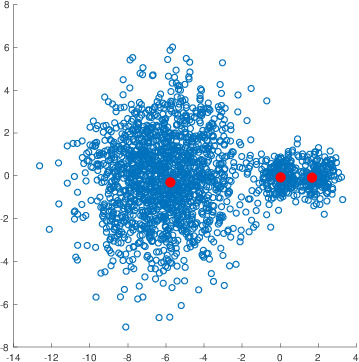}
\subcaption{Learned modes, FSFDPC}
\end{subfigure}
\begin{subfigure}{.24\textwidth}
\includegraphics[width=\textwidth]{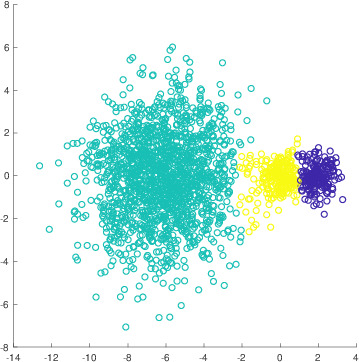}
\subcaption{Learned labels, FSFDPC}
\end{subfigure}
\begin{subfigure}{.24\textwidth}
\includegraphics[width=\textwidth]{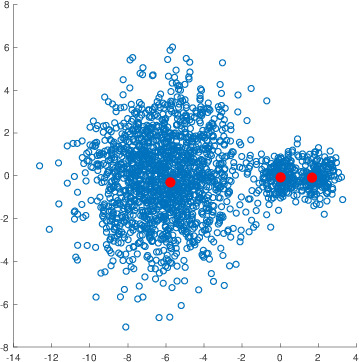}
\subcaption{Learned modes, LUND}
\end{subfigure}
\begin{subfigure}{.24\textwidth}
\includegraphics[width=\textwidth]{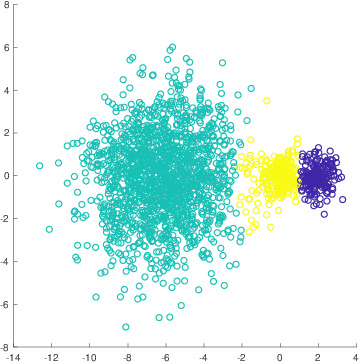}
\subcaption{Learned labels, LUND}
\end{subfigure}
\caption{\label{fig:GaussiansLUNDversusFSFDPC}  For the Gaussian data, both LUND and FSFDPC learn the data modes correctly, and are subsequently able to label the data with high accuracy.  The lack of difference in their comparative performances is attributed to the fact that the data in this case are Gaussians with isotropic covariance matrices, and in particular have simple spherical supports, which confers diffusion distances little advantage compared to Euclidean distances.  LUND used parameters $(\sigma,t)=(1,10^{3})$ for these data.}
\end{figure}

\subsection{Experimental Conclusions}

In all three synthetic examples, LUND performs well.  On the bottleneck data, it gives the same accuracy as spectral clustering with $\psi_{2}$ but better estimates on $\hat{K}$; on the nonlinear data, it gives the best range of accuracies with respect to $\sigma$, while also giving the best estimates of $\hat{K}$; on the Gaussian data, LUND performs as well as FSFDPC while both spectral methods fail.  These results suggest that LUND combines the best properties of spectral clustering with density-based clustering, while enjoying theoretical guarantees.  We remark that extensive experiments with LUND and some variants for real hyperspectral image data were shown by \citet{Murphy2018_Diffusion, Murphy2018_Iterative, Murphy2018_Unsupervised}, demonstrating the competitiveness of LUND with a range of benchmark and state-of-the-art clustering algorithms.  

As shown in Figures \ref{fig:BottleneckAccuracy}, \ref{fig:NonlinearAccuracy}, \ref{fig:GaussiansOA}, \ref{fig:GaussiansAA}, the relationship between $\sigma$ and $t$ is not as simple as suggested by equation (\ref{eqn:TsigmaEquivalence}).  Indeed, in the non-asymptotic case, and in particular in the case of clusters when the underlying density may be empirically 0 in some places, the relationship between $t$ and $\sigma$ does not obey a strict exponential relationship, as suggested by (\ref{eqn:TsigmaEquivalence}).  Instead, $t$ appears to interact with scales \emph{locally} on each cluster, as can be seen by the bifurcations in these plots.  Gaining a complete understanding of the relationship between $\sigma$ and $t$ in the cluster case is a topic of ongoing research.

\subsection{Computational Complexity}

The proposed algorithm enjoys essentially linear computational complexity.  This is achieved through the indexing structure \emph{cover trees}, \citep{Beygelzimer2006cover}, which allows for efficient nearest neighbor searches under the assumption that data has low-dimensional structure.  Indeed, for data $\{x_{i}\}_{i=1}^{n}\subset\mathbb{R}^{D}$, the computation of each of the $n$ data point's $\knn$ nearest neighbors can be achieved at a cost of $O(\knn D C^{d}n\log(n))$, where $d$ is the intrinsic dimension of the data.  This allows for the computation of each point's empirical density estimate $p(x)$ in $O(\knn D C^{d}n\log(n))$, where $\knn$ is the number of nearest neighbors used in the density estimate.  In addition, if the Markov chain is computed not on a fully connected graph, but on a $\knn$ nearest neighbors graph, the cost of computing diffusion distances with the eigenvector approximation is $O(\knn D C^{d}\log(n)+\knn M^{2}n)$, where $M$ eigenvectors are used in the approximation.  In the case that $\knn=O(\log(n)), M=O(1)$, these complexities simplify to $O(DC^{d}\log(n)^{2}n)$.  Computing $\rho_{t}(x)$ is $O(\log(n))$, under the assumption that, except for the $O(\log(n))$ class modes, each point has among its $O(\log(n))$ nearest neighbors a point of higher empirical density.  Subsequent sorting of $p$ and $D_{t}$ are $O(n\log(n))$, so the overall algorithm is $O(DC^{d}\log^{2}(n)n)$, which is linear in $n$, up to logarithmic factors.  
\section{Conclusions and Future Work}\label{sec:Conclusions}

In this article, new methods for bounding diffusion distances, based on nearly-reducible Markov chains, are deployed to provide sufficient conditions under which clustering of data can be guaranteed.  The theoretical results allow to rigorously show that diffusion distances exhibit multitemporal behavior, even in the case that clusters have multiple regions of high-density or nonlinear support.  These estimates on diffusion distance allow to prove performance guarantees on the LUND algorithm.  This may be interpreted as a critique of the popular FSFDPC algorithm, for which theoretical guarantees require unrealistic assumptions on the shapes of the clusters.  Numerical experiments on bottleneck, nonlinear, and Gaussian data indicate that the theoretical results correspond with empirical performance, and that LUND enjoys advantages of both spectral clustering and FSFDPC while tempering their weaknesses.  

While the results presented in this paper for diffusion distances are, we believe, novel and useful for developing performance guarantees for unsupervised learning, they fall short of a full finite sample analysis.  It is of interest to understand how the estimates on $\delta$ and $\lambda_{\numclust+1}$, which govern the estimates of Theorem \ref{thm:MainResult}, and consequently the performance guarantees for clustering, scale with the number of sample points.  Developing such precise estimates would require new mathematical methods for analyzing the spectra of random operators on graphs.  Such an analysis is suggested by recent works in discrete-to-continuum spectral analysis  \citep{Trillos2018error}, though handling the factor $(\I-\Pm_{ii})^{-1}$ may provide for new challenges.  

As remarked in Section \ref{sec:NumericalExperiments}, diffusion distances are Euclidean distances in a new coordinate basis, given by the (right) eigenvectors of $\Pm$, namely $D_{t}^{2}(x,y)=\sum_{\ell=1}^{n}\lambda_{\ell}^{2}(\psi_{\ell}(x)-\psi_{\ell}(y))^{2}.$  A different approach to proving the localization properties of $D_{t}$ with respect to time would be to show that different eigenvectors localize on particular clusters, and show that there are gaps in the eigenvalues $\lambda_{\ell}$ which account for the emergence of mesoscopic equilibria.  This approach is related to the analysis of eigenvectors corresponding to small eigenvalues for the symmetric Laplacian \citep{Shi2009}, and may provide new insights.  

The proposed method also lends itself to the semisupervised setting of \emph{active learning}, in which the user is allowed to query a small number of points for labels.  By estimating which points are most likely to be modes of clusters, the LUND algorithm presents natural candidates to query for labels.  It is of interest to understand if this method may resolve ambiguities for data in which the connections between clusters are weak.  

\section{Acknowledgements}
This project was partially funded by NSF-DMS-125012, NSF-DMS-1724979, NSF-DMS-1708602, NSF-ATD-1737984, AFOSR FA9550-17-1-0280, NSF-IIS-1546392.


\bibliography{DL_JMLR.bib}

\appendix

\section{Proof of Theorem \ref{thm:NearReducibility}}

Notice $\|\Pm^{t}-\Sm^{\infty}\|_{\infty}\le \|\Pm^{t}-\Sm^{t}\|_{\infty}+\|\Sm^{t}-\Sm^{\infty}\|_{\infty}$.  For all $t\ge0$, $\Pm^{t}-\Sm^{t}=\sum_{i=1}^{t}\Sm^{t-i}(\Pm-\Sm)\Pm^{i-1},$ so that $$\|\Pm^{t}-\Sm^{t}\|_{\infty}=\left\|\sum_{i=1}^{t}\Sm^{t-i}(\Pm-\Sm)\Pm^{i-1}\right\|_{\infty}\le \sum_{i=1}^{t}\|\Sm^{t-i}\|_{\infty}\|(\Pm-\Sm)\|_{\infty}\|\Pm^{i-1}\|_{\infty}=t\|(\Pm-\Sm)\|_{\infty}\le t\delta.$$

To bound $\|\Sm^{t}-\Sm^{\infty}\|_{\infty}$, notice that after diagonalizing $\Sm$, 

\[\Sm^{t}=\Z\begin{bmatrix}
    \I_{\numclust}  & \0   \\
   \0  & \D^{t}  \\
\end{bmatrix}\Z^{-1}, \ \
\Sm^{\infty}=\Z\begin{bmatrix}
    \I_{\numclust}  & \0   \\
   \0  & \0 \\
\end{bmatrix}\Z^{-1},\]
where  $\D$ is a diagonal matrix with $\lambda_{\numclust+1}, \lambda_{\numclust+2}, \dots, \lambda_{n}$ on the diagonal.  Hence, $\|\Sm^{t}-\Sm^{\infty}\|_{\infty}\le \|\Z\|_{\infty}\lambda_{\numclust+1}^{t}\|\Z^{-1}\|_{\infty}=\kappa\lambda_{\numclust+1}^{t}$, as desired.  The second result of the theorem follows similarly.  

\end{document}